\documentclass[10pt]{article} % For LaTeX2e
\usepackage[accepted]{tmlr}
\definecolor{mygray}{gray}{.9}

\usepackage{color, colortbl}
\usepackage[lined,boxed,commentsnumbered,ruled]{algorithm2e}
\usepackage{algorithmic}
\newcommand{\defeq}{\mbox {$  \ \stackrel{\Delta}{=} $}}
\usepackage{subfig}
\usepackage{wrapfig}
\usepackage{amsmath,amsfonts,bm}

\def\eqref#1{equation~\ref{#1}}
\def\1{\bm{1}}

\DeclareMathAlphabet{\mathsfit}{\encodingdefault}{\sfdefault}{m}{sl}
\SetMathAlphabet{\mathsfit}{bold}{\encodingdefault}{\sfdefault}{bx}{n}

\usepackage{amsmath}
\usepackage{amssymb}
\usepackage{mathtools}
\usepackage{amsthm}
\usepackage{stfloats}
\usepackage{xcolor}
\usepackage{multirow}
\usepackage{xcolor}
\usepackage{booktabs}      
\usepackage{mwe}
\usepackage{hyperref}
\usepackage[capitalize,noabbrev]{cleveref}
\usepackage{url}
\theoremstyle{definition}
\newtheorem{theorem}{Theorem}

\newtheorem{lemma}{Lemma}

\newtheorem{definition}{Definition}

\newtheorem{remark}{Remark}

\title{Towards Undistillable Models by Minimizing Conditional Mutual Information}

\author{\name Linfeng Ye\thanks{Authors contributed equally.} \email linfeng.ye@mail.utoronto.ca\\
      Edward S. Rogers Sr. Department of Electrical and Computer Engineering\\
      University of Toronto
      \AND
      \name Shayan Mohajer Hamidi\footnotemark[1] \email smohajer@stanford.edu \\
      Department of Electrical Engineering\\
      Stanford University
      \AND
      \name En-Hui Yang \email ehyang@uwaterloo.ca\\
      Department of Electrical and Computer Engineering\\
      University of Waterloo}

\def\month{06}  % Insert correct month for camera-ready version
\def\year{2025} % Insert correct year for camera-ready version
\def\openreview{\url{https://openreview.net/forum?id=jVABSsD4Vf}} % Insert correct link to OpenReview for camera-ready version

\begin{document}

\maketitle

\begin{abstract}
A deep neural network (DNN) is said to be undistillable if, when used as a black-box input-output teacher, it cannot be distilled through knowledge distillation (KD). In this case, the distilled student (referred to as the knockoff student) does not outperform a student trained independently with label smoothing (LS student) in terms of prediction accuracy. To protect intellectual property of DNNs, it is desirable to build undistillable DNNs. To this end, it is first observed that an undistillable DNN may have the trait that each cluster of its output probability distributions in response to all sample instances with the same label should be highly concentrated to the extent that each cluster corresponding to each label should ideally collapse into one probability distribution. Based on this observation and by measuring the concentration of each cluster in terms of conditional mutual information (CMI), a new training method called CMI minimized (CMIM) method is proposed, which trains a DNN by jointly minimizing the conventional cross entropy (CE) loss and the CMI values of all temperature scaled clusters across the entire temperature spectrum. The resulting CMIM model is shown, by extensive experiments, to be undistillable by all tested KD methods existing in the literature. That is, the knockoff students distilled by these KD methods from the CMIM model underperform the respective LS students. In addition, the CMIM model is also shown to performs better than the model trained with the CE loss alone in terms of their own prediction accuracy. The code for the paper is publicly available at \url{https://anonymous.4open.science/r/CMIM-605C/README.md}. 
\end{abstract}

\section{Introduction} \label{sec:intro}

Originally aiming for model compression, knowledge distillation \citep{buciluǎ2006model,hinton2015distilling} (KD) has received significant attention from both academia and industry in recent years due to its remarkable effectiveness. The essence of KD is to transfer the knowledge of a pre-trained
large model (teacher) to a smaller model (student). Building on the work of \citet{hinton2015distilling}, numerous follow-up works have endeavored to enhance the performance of KD \citep{romero2014fitnets,anil2018large,park2019relational} and to gain deeper insights into why distillation is effective \citep{phuong2019towards,mobahi2020self,ye2024bayes,allen2020towards,menon2021statistical,borup2021even}.

In the scenario where the teacher does not want its knowledge to be transferred, however, KD is undesirable and indeed poses a threat to intellectual property (IP) of the teacher \citep{shokri2015privacy}. Developing and training a high-quality large DNN requires significant investments of time, effort, finances, and resources, including extensive data annotation and computational infrastructure. Developers of the large DNN may want to prevent the knowledge of the large DNN from being transferred by their competitors. However, once the large DNN is released as a ``black box'', anyone can apply a logit-based KD method (or equivalently a distribution-based KD method \citep{zheng2024knowledge}) to distill the DNN as a teacher. The goal is to train a student, referred to as a knockoff student in the context of DNN IP protection, that mimics the teacher's behavior to gain competitive advantages. As such, in this case, it would be desirable for the developers to build the large DNN so that it is undistillable.  The question, of course, is how. 

Before delving deeper into the above question, let us first clarify what we mean by saying that a DNN is undistillable. At this point, we invoke the concept of distillable DNN introduced recently in \citet{yang2024markov}:

\begin{definition} \label{def:Distillability}
[Distillability of a DNN \citep{yang2024markov}] When used as a black-box input-output teacher, a DNN is said to be distillable\footnote{There are two reasons for us to adopt this definition of distillability. First,  as shown theoretically in \citet{zheng2024knowledge}, knowledge distillation reduces to label smoothing in the limit as the temperature approaches infinity. Therefore, if the teacher is distillable, the distilled student should perform no worse than the LS student. Second, if a knockoff student cannot outperform the LS student, then there is no incentive for the teacher to be leveraged since the LS student can be trained on its own.} with respect to a student if there exists a KD method which, when applied to the teacher and student, yields a knockoff student outperforming the student trained alone with label smoothing (LS student) in terms of prediction accuracy. 
\end{definition}

% \vspace{-1cm}
Therefore, a DNN is undistillable if no knockoff student can outperform the respective LS student regardless of which logit-based KD method is used. Since there are so many logit-based KD methods and so many students, what type of a DNN is undistillable? In comparison with the training of LS student, \Cref{def:Distillability} provides some insight into a trait  that an undistillable DNN may possess. For each label, consider the cluster of the output probability distributions of the DNN in response to all input sample instances with that label. If each cluster corresponding to each label is highly concentrated to the extent that all probability distributions within the cluster more or less collapse into one probability distribution, then the student training within KD is similar to that of the respective LS student regardless of which logit-based KD method and which student are applied. In this case, one would expect that no knockoff student would perform significantly better than the respective LS student. Therefore, a DNN possessing this trait will likely be undistillable.

From a theoretical perspective, KD relies on the diversity within the teacher’s predicted distributions to convey richer contextual information to the student. The greater the diversity in the cluster of predictions for a given class, the more nuanced guidance the student can receive (see \citet{ye2024bayes}). Conversely, when the teacher’s outputs become highly concentrated---meaning predictions for inputs with the same label converge to a single probability vector---this contextual richness is lost. The student, in this case, receives largely redundant information, rendering the distillation process as ineffective as direct training with label smoothing.

To quantify this notion of concentration, we turn to conditional mutual information (CMI) \citep{10900607}. Specifically, let $X$ denote the random input sample to the DNN, and $Y$ be the ground truth label of $X$. Let $\hat{Y}$ denote  the random label predicted by the DNN in response to input $X$. It was shown in \citet{10900607} that for each label $y$, the label specific CMI $\mathrm{I}(X;\hat{Y}|~Y=y)$ measures the concentration of the cluster corresponding to label $y$, and the CMI $\mathrm{I}(X;\hat{Y}|~Y)$ measures the average concentration across all clusters. To build an undistillable DNN, one then is motivated to minimize jointly the conventional cross entropy (CE) loss and the CMI $\mathrm{I}(X;\hat{Y}|~Y)$.

In this paper, we will go one step further. In KD \citep{hinton2015distilling}, temperature scaling of logits is often applied. It was shown in \citet{zheng2024knowledge} that logit temperature scaling with temperature $T$ can be equivalently achieved by power transform of the output probability distribution with power $\alpha = 1/T$. Further, it was demonstrated in \citet{ye2024bayes} that the purpose of temperature scaling or power transform is to enlarge the CMI values of temperature scaled (or power transformed) clusters, and enlarging CMI values in turn improves the performance of distilled students. Since here we want to achieve the opposite, we want to make sure that all CMI values of all power transformed clusters can be made small. To this end, we further extend  the label specific CMI $\mathrm{I}(X;\hat{Y}|~Y=y)$  and the CMI $\mathrm{I}(X;\hat{Y}|~Y)$ to $\mathrm{I}(X;\hat{Y}^{\alpha} |~Y=y)$  and  $\mathrm{I}(X;\hat{Y}^{\alpha[Y]}|~Y)$, respectively,  so that $\mathrm{I}(X;\hat{Y}^{\alpha} |~Y=y)$ measures the concentration of the power transformed cluster corresponding to label $y$ with power $\alpha$, and $\mathrm{I}(X;\hat{Y}^{\alpha[Y]}|~Y)$ measures the average concentration across all power transformed clusters with power $\alpha[Y]$, where different clusters may be power transformed with different power $\alpha$. 
Notably, allowing separate temperature scaling (i.e., power transformation) for each class is essential, as an adaptive adversary may apply class-specific scaling to selectively alter cluster concentration and improve the knockoff student’s performance.

Based on the above discussion and towards building undistillable DNNs, we then propose a new training method called  CMI minimized method, which trains a DNN by jointly minimizing the CE loss and all CMI values of all power transformed clusters, i.e., jointly minimizing the CE loss and $\mathrm{I}(X;\hat{Y}^{\alpha[Y]}|~Y), ~ \forall \alpha[Y] >0$. 

The resulting trained DNN is referred to as the CMI minimized (CMIM) DNN.  The contributions of the paper are summarized as follows:

\noindent $\bullet$ An insight is provided that in order for a DNN to be undistillable, it is desirable for the DNN to possess the trait that each cluster of the DNN's output probability distributions corresponding to each label is highly concentrated to the extent that all probability distributions within the cluster more or less collapse into one probability distribution close to the one-hot probability vector of that label.

\noindent $\bullet$ We extend  the label specific CMI $\mathrm{I}(X;\hat{Y}|~Y=y)$  and the CMI $\mathrm{I}(X;\hat{Y}|~Y)$ to $\mathrm{I}(X;\hat{Y}^{\alpha} |~Y=y)$  and  $\mathrm{I}(X;\hat{Y}^{\alpha[Y]}|~Y)$, respectively,  so that $\mathrm{I}(X;\hat{Y}^{\alpha} |~Y=y)$ measures the concentration of the power transformed cluster corresponding to label $y$ with power $\alpha$, and $\mathrm{I}(X;\hat{Y}^{\alpha[Y]}|~Y)$ measures the average concentration across all power transformed clusters with power $\alpha[Y]$, where different clusters may be power transformed with different power $\alpha$.

\noindent $\bullet$ We develop a novel training method dubbed CMI minimized method to train a DNN by jointly minimizing the CE loss and all CMI values of all power transformed clusters with the resulting trained DNN referred to as the CMIM DNN. 

\noindent $\bullet$ To the best of our knowledge, our method is the first in the literature capable of training undistillable DNNs that remain robust against a wide range of KD methods. Furthermore, for the notion of undistillability, we are the first to employ the formal definition introduced in \citet{yang2024markov}.

\noindent $\bullet$ We show, by extensive experiments over three popular image classification datasets, namely CIFAR-100 \citep{CIFAR100}, TinyImageNet \citep{tinyiamgenet} and ImageNet \citep{imagenet}, that CMIM DNNs have very small CMI values and are indeed undistillable by all tested KD methods existing in the literature. That is, the knockoff students distilled by these KD methods from the CMIM models underperform the respective LS students. On the other hand, models trained by defense training methods proposed in the literature are all distillable. 

\noindent $\bullet$ In addition, we show that the CMIM models achieve a higher classification accuracy compared to those trained with the conventional CE loss.

\section{Related Works} \label{sec:related}
In this section, we mention some defense methods against  the threat posed by knockoff students attempting to steal the IP of pre-trained DNNs via logit-based KD methods. For a thorough review of related works, including detailed discussions about recent logit-based KD methods, please refer to \Cref{sec:app_rel}. These defense methods can be mainly categorized into two groups: (i) model stealing resistant training methods that specifically train DNNs to reduce the accuracy of knockoff students while maintaining the original classification accuracy of the model \citep{nasty,SNT}; and (ii) post-training defense methods that perform minimal perturbations to the pre-trained model's predictions to mislead the knockoff student \citep{RSP,MAD,APGP}. Nonetheless, in \Cref{sec:exp}, we will show that models trained by all these defense methods  are indeed distillable.

\section{Notation and Preliminaries}

\subsection{Notation}

The set of real numbers is denoted by $\mathbb{R}$. Vectors are denoted by bold-face  letters (e.g., $\boldsymbol{w}$). The $i$-th element of vector $\boldsymbol{w}$ is denoted by $\boldsymbol{w}[i]$. For two vectors $\boldsymbol{u},\boldsymbol{v} \in \mathbb{R}^C$, the inequality $\boldsymbol{u} \leq \boldsymbol{v}$ implies that $\boldsymbol{u}[i] \leq \boldsymbol{v}[i]$, $\forall i \in [C]$. For a positive integer $K$, let $[K]\triangleq \left\{1,...K\right\}$. Assume that there are $C$ class labels with $[C]$ as the set of class labels. Let ${\cal P}([C])$ denote the set of all $C$ dimensional probability distributions. For any two probability distributions $P_1, P_2 \in {\cal P}([C])$, the CE and Kullback-Leibler (KL) divergence between $P_1$ and $P_2$ are denoted by $\mathsf{H}(P_1, P_2)$ and $\mathrm{KL}(P_1, P_2)$, respectively. For any $y\in[C]$ and $P\in {\cal P}([C])$, write the CE of the one-hot probability distribution corresponding to $y$ and $P$ as $\mathsf{H}(y, P)$. 

For any differentiable function $ f(\cdot)$, $\nabla_{\boldsymbol{w}} f(\cdot)$ denotes its gradient vector w.r.t. vector $\boldsymbol{w}$. 

% The directional derivative of $f(\cdot)$ in direction of $\boldsymbol{\gamma}$ is denoted by $\mathsf{D}_{\boldsymbol{\gamma}}(\cdot)$

For any pair of random variables $(X,Y)$, denote its joint probability distribution by $P_{X, Y}(x,y)$ or simply $P(x,y)$ whenever there is no ambiguity, the marginal distribution of $Y$ by $P_Y (y)$, and the conditional distribution of $Y$ given $X=x$ by $P_{Y|X}(\cdot|x)$. 
The mutual information between two random variables $X$ and $Y$ is denoted by $\mathrm{I}(X,Y)$, and the CMI of $X$ and $Y$ given a third random variable $Z$ is 
$\mathrm{I}(X,Y|Z)$.

We regard a classification  DNN as a mapping from raw data $x \in \mathbb{R}^{d}$ to a probability distribution $q_x \in {\cal P} ([C])$. Given a  DNN:  $x \in \mathbb{R}^{d} \to  q_x$, let $\boldsymbol{\theta}$ denote its weight vector consisting of all its connection weights; whenever there is no ambiguity, we also write $q_x$ as $q_{x, \boldsymbol{\theta}}$. 

\subsection{Label Smoothing} \label{sec:LS}
Label smoothing (LS) \citep{pereyra2017regularizing} is a regularization technique that prevents peaked output probability distributions, leading to better generalization, by minimizing the objective function:
\begin{align}
{\cal L}^{LS} = (1-\epsilon)\mathsf{H}(y, q_x) + \epsilon \mathsf{H}(u, q_x),
\end{align}
where $u$ is the uniform distribution over $C$ classes, and $\epsilon$ controls the strength of the regularization.

\subsection{Power Transform of Probability Distribution}
In a ``black-box'' teacher setting, where only the output probability vectors (and not the logits) of the teacher are accessible to the public, applying temperature scaling directly over the logits of the teacher is not feasible in training knockoff students. In this case, KD training can resort to applying ``power transformation of probability distribution'' directly to the output probability vectors \citep{zheng2024knowledge}. Specifically, given  $P\in {\cal P([C])}$, and a non-negative real number $\alpha$, the power transform of $P$ is another probability distribution define as
\begin{align}
P^{\alpha} [i]= \frac{(P[i])^\alpha}{\sum_{j\in [C]}(P[j])^\alpha},\quad \forall i \in [C].    
\end{align}

It is not hard to verify that the power transformed probability distribution $P^{\alpha}$ is equal to the softmax of  the logits scaled by temperature $T = 1/\alpha$. Therefore, temperature scaling can be equivalently operated directly on the output probability distribution through power transform.

\subsection{CMI value of a DNN}
As discussed in \citet{10900607}, for a classifier $f: x \in \mathbb{R}^{d} \to  q_x$, let $\hat{Y}$ be the random label predicted by the $f$ with probability $q_X[\hat{Y}]$ in respond to the input $X$. For each cluster corresponding to label $y \in [C]$, we have
% \small
\begin{align}
&\mathrm{I}(X; \hat{Y}|Y=y) = \sum_x P_{X|Y}(x|y)\left [\sum_{i=1}^C P_{\hat{Y}|XY} ( \hat{Y}=i |x, y) \ln \frac{P_{\hat{Y}|XY} (\hat{Y}=i|x, y)}{P_{\hat{Y}|Y}(\hat{Y}=i|Y=y)}\right ]\\
 & \qquad = \mathbb{E}_{X|Y}\left [ \left ( \sum_{i=1}^C q_X[i] \ln \frac{q_X[i]}{P_{\hat{Y}|Y}(\hat{Y}=i|Y=y)}\right ) |Y=y\right] =\mathbb{E}_{X|Y} \left[\mathrm{KL}\left( q_X, s_y \right)|Y=y\right],\label{eq:CMI1}
\end{align}
% \normalsize
where  $ P_{\hat{Y}|XY} ( \hat{Y}=i |x, y) = q_x [i] $ follows from the Markov chain $Y \to X \to \hat{Y}$, and 
$s_y = P_{\hat{Y}|Y}(\cdot |y) = \mathbb{E}_{X|Y}\left[ q_X|Y=y\right]$. 
$ \mathrm{I}(X; \hat{Y}|Y=y)$ measures the concentration of the cluster corresponding to label $y \in [C]$. Averaging over all clusters corresponding to all labels $y$, we get 
 \begin{align}
    & \mathrm{I}(X; \hat{Y}|Y) = \sum_{y \in [C]} P_Y (y) \mathrm{I}(X; \hat{Y}|Y=y) = \mathbb{E}_{XY} \left[\mathrm{KL}\left( q_X, s_Y \right) \right].\label{eq:CMI1-2}
 \end{align}
$ \mathrm{I}(X; \hat{Y}|Y)$ measures the average concentration across all clusters.

When the distribution $P_{X,Y}$ is unknown, we can approximate the CMI of $f$ by its empirical value from a data sample (a training dataset or mini-batch thereof) ${\cal D} = \{ (x_i, y_i) \}_{i=1}^m$. To this end, let ${\cal D}_y = \left\{ 1\leq j \leq m: y_j=y \right\}$. Denote the size of ${\cal D}_y $ by $ |{\cal D}_y |$.  The empirical values of each label specific CMI and the CMI can be calculated as follows
% \small
\begin{align}
&\mathrm{I}^{emp}(X; \hat{Y}|Y=y) = \frac{1}{|{\cal D}_y |  }  \sum_{i \in {\cal D}_y} \mathrm{KL}(q_{x_i}, s_y^{emp}),\\
& \mathrm{I}^{emp}(X; \hat{Y}|Y) = \frac{1}{m  }  \sum_{i=1}^m  \mathrm{KL}(q_{x_i}, s_{y_i}^{emp}), \\
\text{where\quad} &s_y^{emp} = \frac{1}{|{\cal D}_y|} \sum_{i \in {\cal D}_y } q_{x_i},  \forall y \in [C]. \label{eq:emp_CMI}
\end{align} 
% \vspace{-0.5in}

\section{CMI Minimized Method}

In this section, we present our CMI minimized method. We begin with extending $ \mathrm{I}(X; \hat{Y}|Y=y)$  and $ \mathrm{I}(X; \hat{Y}|Y)$ to the case of power transformed clusters. 

\subsection{Information Quantities for Power Transformed Clusters} \label{sec:motiv}

Consider a classification DNN: $f: x \in \mathbb{R}^{d} \to  q_x$ which maps  input sample instances $x$ with different labels into clusters of probability distributions $q_x$ in the space ${\cal P} ([C])$, with one cluster per label. For each label $y \in [C]$, apply the power transform with power $\alpha$ to each probability distribution $q_x$ within the cluster corresponding to the label $y$. Then, we obtain a power transformed cluster. To measure the concentration of the power transformed cluster, we extend $ \mathrm{I}(X; \hat{Y}|Y=y)$ to the following information quantity
 % \small
\begin{align}
&\mathrm{I}(X; \hat{Y}^{\alpha} |Y=y)  =\mathbb{E}_{X|Y} \left[\mathrm{KL}\left( q_X^{\alpha}, s_{y,\alpha} \right)|Y=y\right],\label{eq:CMI1-p}
\end{align}
% \normalsize
where $s_{y, \alpha} =  \mathbb{E}_{X|Y}\left[ q_X^{\alpha}|Y= y\right]$. Note that if we regard $ \hat{Y}^{\alpha} $ as the random label predicted by $f$ with probability $q_X^{\alpha} ( \hat{Y}^{\alpha}) $ in response to the input sample $X$, i.e., given $X$, $ \hat{Y}^{\alpha} $ is equal to a label $c$ with probability $q_X^{\alpha} (c ) $, $\forall c \in [C]$, then $ \mathrm{I}(X; \hat{Y}^{\alpha} |Y=y) $ is exactly the CMI between $X$ and $ \hat{Y}^{\alpha} $ given $Y =y$. Thus, $ \mathrm{I}(X; \hat{Y}^{\alpha} |Y=y) $ measures the concentration of the power transformed cluster corresponding to $y$. 

Now, we go one step further and allow different clusters to be power transformed with different powers. Suppose that the cluster corresponding to label $y$ is power transformed with power $\alpha [y]$. Let $ \hat{Y}^{\alpha [Y]} $ be the random label predicted by $f$ with probability $q_X^{\alpha[Y]} ( \hat{Y}^{\alpha[Y]}) $ in response to the input sample $X$ given $Y$. That is, given $Y=y$ and $X=x$, $ \hat{Y}^{\alpha [Y]} $ is equal to $c$ with probability $  q_x^{\alpha[y]} (c)$ for any $c \in [C]$. We can then extend $ \mathrm{I}(X; \hat{Y}|Y)$ to $ \mathrm{I}(X; \hat{Y}^{\alpha[Y]}|Y)$
 % \small
\begin{align}
\mathrm{I}(X; \hat{Y}^{\alpha[Y]} |Y)  & =\mathbb{E}_{XY} \left[\mathrm{KL}\left( q_X^{\alpha[Y]}, s_{Y, \alpha[Y]} \right)\right],\label{eq:CMI2-p} \\
     \qquad & = \sum_{y \in [C]} P_Y (y) \left [   \mathbb{E}_{X|Y} \left[\mathrm{KL}\left( q_X^{\alpha[y]}, s_{y, \alpha[y]} \right)|Y=y\right]     \right ] \label{eq:CMI3-p} \\
     \qquad & = \sum_{y \in [C]} P_Y (y) \mathrm{I}(X; \hat{Y}^{\alpha[y]} |Y=y),  \label{eq:CMI4-p}
\end{align}
% \normalsize
where for each $y \in [C]$, 
\begin{align}
     s_{y,\alpha[y]} &= P_{ \hat{Y}^{\alpha[Y]} |Y} (\cdot |y) = \sum_{x} P_{X|Y} (x | y) q_x^{\alpha[y]} =   \mathbb{E}_{X|Y}\left[ q_X^{\alpha[y]}|Y= y\right]. \label{eq:CMI6-p}
\end{align}
Note that $ \mathrm{I}(X; \hat{Y}^{\alpha[Y]}|Y)$ is exactly the CMI between $X$ and $\hat{Y}^{\alpha[Y]}$ given $Y$ and measures the average concentration across all power transformed clusters with power function $\alpha[Y]$.
However, $Y$, $X$, and $ \hat{Y}^{\alpha[Y]}  $ do not form a Markov chain anymore.

When the distribution $P_{X,Y}$ is unknown, we can approximate $ \mathrm{I}(X; \hat{Y}^{\alpha[Y]}|Y=y)$ and $ \mathrm{I}(X; \hat{Y}^{\alpha[Y]}|Y)$ by their respective  empirical values from a data sample (a training dataset or mini-batch thereof) ${\cal D} = \{ (x_i, y_i) \}_{i=1}^m$: 
\begin{align}
&\mathrm{I}^{emp}(X; \hat{Y}^{\alpha[Y]} |Y=y) = \frac{1}{|{\cal D}_y |  }  \sum_{i \in {\cal D}_y} \mathrm{KL}(q_{x_i}^{\alpha[y]}, s_{y, \alpha[y]}^{ emp}), \label{eq:emp_CMI1-p} \\
& \mathrm{I}^{emp}(X; \hat{Y}^{\alpha[Y]} |Y) = \frac{1}{m  }  \sum_{i=1}^m  \mathrm{KL}(q_{x_i}^{\alpha[y_i]}, s_{y_i, \alpha[y_i]}^{ emp}), \label{eq:emp_CMI2-p}\\
\text{where\quad} & s_{y, \alpha[y]}^{ emp} = \frac{1}{|{\cal D}_y|} \sum_{i \in {\cal D}_y } q_{x_i}^{\alpha[y]},  \forall y \in [C]. \label{eq:emp_CMI3-p}
\end{align} 

\begin{wrapfigure}{h}{0.34\textwidth}
\begin{center}
% \vskip -0.3in
\includegraphics[width=0.33\textwidth]{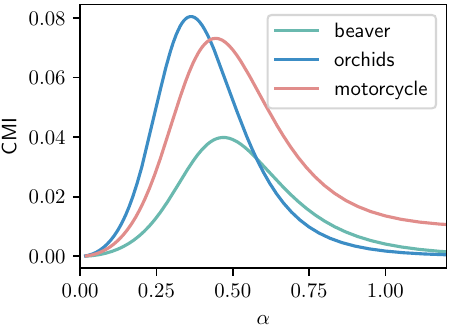}
\end{center}
% \vskip -0.2in
\caption{The CMI value $\mathrm{I}(X; \hat{Y}^{\boldsymbol{\alpha}[Y]}|Y=y)$ for three randomly selected classes $y=\{\text{beaver, orchids, motorcycle}\}$ vs the power transform factor $\alpha$; the model is ResNet-50 pre-trained  on CIFAR-100.  As observed, as $\alpha$ grows, the CMI value becomes larger, peaks, and then gradually becomes smaller.}
% \vskip -0.9in
\label{fig:CmiPerClassVsFactor}
\end{wrapfigure}

As discussed in \Cref{sec:intro}, an undistillable DNN should exhibit the trait that each of these clusters is highly concentrated and ideally collapses into a single probability distribution that closely resembles the one-hot probability vector for that label.

\begin{remark}
While we leverage the concept of CMI from \citet{10900607}, the way it is calculated in our work significantly differs from how it is calculated in \citet{10900607}. In \citet{10900607}, $ \mathrm{I}(X; \hat{Y}|Y)$ is calculated under the assumption of a Markov chain $Y \to X \to \hat{Y}$. In contrast, we quantify cluster compactness using $\mathrm{I}(X; \hat{Y}^{\alpha[Y]} |Y)$, where $\hat{Y}^{\alpha[Y]}$ explicitly depends on $Y$, violating the Markov assumption. Moreover, minimizing $\mathrm{I}(X; \hat{Y}^{\alpha[Y]} |Y)$ over all possible values of $\alpha$ introduces additional challenges, which we address in the next subsection.
\end{remark}

\subsection{Framework for Minimizing CMI Values of Power Transformed Clusters} \label{sec:method}

Towards building an undistillable DNN, we now train a DNN $f: x \in \mathbb{R}^{d} \to  q_x$ by jointly minimizing the CE loss and all CMI values of all power transformed clusters. Let 
 \[ \boldsymbol{\alpha} = \big[\alpha[1], \alpha[2], \dots, \alpha [C] \big], \]
 and write each $q_x$ as $q_{x, \theta}$. In our CMI minimized method, the objective function we want to minimize is
  \begin{align} 
&  \mathbb{E}_{XY} \Big[\mathsf{H}(Y, q_{X,  \boldsymbol{\theta}})\Big] + \lambda\max_{\boldsymbol{\alpha}} \mathrm{I}(X; \hat{Y}^{\boldsymbol{\alpha}[Y]}|Y) , \label{eq:obj1}
\end{align}
where $\lambda>0$ is a hyper-parameter trading the CE loss with the maximum CMI, and the maximization over $\boldsymbol{\alpha}$ is taken over the region $0 \leq \alpha [i] \leq \beta$, $1 \leq i \leq C$. The optimization problem then becomes 
 \begin{align} 
& \min_{\boldsymbol{\theta}} \Big\{ \mathbb{E}_{XY} \Big[\mathsf{H}(Y, q_{X,  \boldsymbol{\theta}})\Big] + \lambda\max_{\boldsymbol{\alpha}} \mathrm{I}(X; \hat{Y}^{\boldsymbol{\alpha}[Y]}|Y) \Big \}  \nonumber \\
= &\min_{\boldsymbol{\theta}} \Big\{ \mathbb{E}_{XY} \Big[\mathsf{H}(Y, q_{X,  \boldsymbol{\theta}})\Big] + \lambda\max_{\boldsymbol{\alpha}} \sum_y P_{Y} [y] \mathrm{I}(X; \hat{Y}^{\alpha[y]}|Y=y) \Big \} \label{eq:unconstrained1} \\
=& \min_{\boldsymbol{\theta}} \Big\{ \mathbb{E}_{XY} \Big[\mathsf{H}(Y, q_{X,  \boldsymbol{\theta}})\Big] + \lambda \sum_y P_{Y} [y] \max_{\alpha[y]} \mathrm{I}(X; \hat{Y}^{\alpha[y]}|Y=y) \Big \} . \label{eq:unconstrained2}
%\\
%& = \min_{\boldsymbol{\theta}} \Big\{ \mathbb{E}_{XY} \Big[\mathsf{H}(Y, q_{X,  \boldsymbol{\theta}})\Big] + \lambda \sum_y P_{Y} [y] \max_{\alpha[y]} %\mathbb{E}_{X|Y} \big[ \mathrm{KL}\left( q_{X,  \boldsymbol{\theta}}^{\alpha[y] }, s_{y,\alpha[y] } \right) \left | \right. Y=y \big]  \Big \} .\label{eq:unconstrainedObj}
\end{align}

In order to get a better understanding about the behavior of the second term in the objective function of \eqref{eq:unconstrained1} w.r.t. $\boldsymbol{\alpha}$,  we depict in \Cref{fig:CmiPerClassVsFactor} $\mathrm{I}(X; \hat{Y}^{\boldsymbol{\alpha}[Y]}|Y=y)$ vs $\alpha[y]$ for three randomly-selected  classes $y$ using a pre-trained ResNet-50 on CIFAR-100.  In \Cref{fig:CmiPerClassVsFactor},  $\max_{\alpha} ~ \mathrm{I}(X; \hat{Y}^{\alpha}|Y=y)   $  is achieved at a value of $\alpha$ which is between $0.25$ and $0.75$. In Theorem~\ref{th:deriv} of  \Cref{app:proofderiv}, we further show that for each label $y$, $ \mathrm{I}(X; \hat{Y}^{\alpha}|Y=y)   $ as a function of $\alpha$ is continuously differentiable.

However, finding an algorithmic solution to the min-max problem in \eqref{eq:unconstrained1} to \eqref{eq:unconstrained2} is challenging. To overcome this difficulty, we next develop a more tractable expression for 
$ \max_{\alpha} ~ \mathrm{I}(X; \hat{Y}^{\alpha}|Y=y)  $. At this point, we invoke the following theorem, which will be proved in \Cref{app:theorem2}.
\begin{theorem} \label{th2}
For any label $y$, 
 \begin{equation}
     \max_{\alpha} ~ \mathrm{I}(X; \hat{Y}^{\alpha}|Y=y)  = \lim_{\omega \to \infty} \frac{1}{\omega} \ln \frac{1}{\beta} \int_0^{\beta} \exp{ \{\omega 
      \mathrm{I}(X; \hat{Y}^{\alpha}|Y=y) \} } d \alpha . \label{eq:th2}
 \end{equation}   
\end{theorem}
Therefore, when $\omega$ is large, $ \max_{\alpha} ~ \mathrm{I}(X; \hat{Y}^{\alpha}|Y=y)  $ can be approximated by 
\begin{align}
 \max_{\alpha} &~ \mathrm{I} (X; \hat{Y}^{\alpha}|Y=y) \approx \frac{1}{\omega} \ln \frac{1}{\beta} \int_0^{\beta} \exp{ \{\omega 
      \mathrm{I}(X; \hat{Y}^{\alpha}|Y=y) \} } d \alpha  \label{eq:th2-1} \\
      & \approx \frac{1}{\omega} \ln \left [ \frac{1}{N} \sum_{i=1}^N   \exp{ \{\omega 
      \mathrm{I}(X; \hat{Y}^{\alpha_i}|Y=y) \} }         \right  ], \label{eq:th2-2}
\end{align}
where $N$ is relatively large, and $\alpha_i = i\beta /N$.

Now plugging \eqref{eq:th2-2} into \eqref{eq:unconstrained2}, we have 
\begin{align}
 &\min_{\boldsymbol{\theta}} \Big\{ \mathbb{E}_{XY} \Big[\mathsf{H}(Y, q_{X,  \boldsymbol{\theta}})\Big] +  \frac{\lambda}{\omega} \sum_y P_{Y} [y] 
 \ln \left [ \frac{1}{N} \sum_{i=1}^N   \exp{ \{\omega 
      \mathrm{I}(X; \hat{Y}^{\alpha_i}|Y=y) \} }         \right  ] \Big \} .\label{eq:th2-3} 
\end{align}
Note that the second term in  the objective function of \eqref{eq:th2-3} is not amenable to parallel computation via GPU due to the dependency of KL divergence on $  s_{y, \alpha_i} $, the centroid of the power transformed cluster corresponding to $Y=y$ with power $ \alpha_i  $. To get around this difficulty, we follow  the approach in \citet{10900607} and introduce  dummy distributions $Q_{y,i} \in {\cal P} ([C])$ for each $(y, i)$
to rewrite $  \mathrm{I}(X; \hat{Y}^{\alpha_i}|Y=y) $ as follows
\begin{align}
\mathrm{I}(X; \hat{Y}^{\alpha_i}|Y=y) &= \mathbb{E}_{X|Y} \big[ \mathrm{KL}\left( q_{X,  \boldsymbol{\theta}}^{\alpha_i }, s_{y,\alpha_i } \right) \left | \right. Y=y \big] \nonumber \\
& = \min_{Q_{y, i}   } \mathbb{E}_{X|Y} \big[ \mathrm{KL}\left( q_{X,  \boldsymbol{\theta}}^{\alpha_i }, Q_{y, i} \right) \left | \right. Y=y \big],
\label{eq:th2-4}
\end{align}
where the minimum in the above is achieved when 
   \begin{align}
       Q_{y, i} = s_{y, \alpha_i} = \mathbb{E}_{X|Y}\left[ q_{X, \boldsymbol{\theta}}^{\alpha_i}|Y= y\right]. \label{eq:th2-4+}
   \end{align}
Combining \eqref{eq:th2-4} with \eqref{eq:th2-3}, we are led to solve the double minimization problem \footnote{In practice, solving the double minimization problem introduces only minor runtime overhead, as the inner optimization problem has an analytic solution and can be parallelized efficiently on modern GPUs. To demonstrate the efficiency of CMIC, we report the wall-clock training time and compare the computational overhead of CMIM and CE in \cref{sec:computation}.}
\begin{align}
&  \min_{\boldsymbol{\theta}} \Big\{ \mathbb{E}_{XY} \Big[\mathsf{H}(Y, q_{X,  \boldsymbol{\theta}})\Big] +  \frac{\lambda}{\omega} \sum_y P_{Y} [y] 
 \ln \left [ \frac{1}{N} \sum_{i=1}^N   \exp{ \{\omega \min_{Q_{y, i}   }\mathbb{E}_{X|Y} \big[ \mathrm{KL}\left( q_{X,  \boldsymbol{\theta}}^{\alpha_i }, Q_{y,i } \right) \left | \right. Y=y \big] \} }         \right  ] \Big \}  \label{eq:th2-5} \\
 & =  \min_{\boldsymbol{\theta}} \min_{\{Q_{y, i}\}_{y\in [C], i \in [N]}} \Big\{ \mathbb{E}_{XY} \Big[\mathsf{H}(Y, q_{X,  \boldsymbol{\theta}})\Big]   \nonumber \\
& \qquad \qquad + \frac{\lambda}{\omega} \sum_y P_{Y} [y] 
 \ln \left [ \frac{1}{N} \sum_{i=1}^N   \exp{ \{\omega \mathbb{E}_{X|Y} \big[ \mathrm{KL}\left( q_{X,  \boldsymbol{\theta}}^{\alpha_i }, Q_{y,i } \right) \left | \right. Y=y \big] \} }         \right  ] \Big \}  \label{eq:th2-5} 
\end{align}

When the distribution $P_{X,Y}$ is unknown, it can be approximated by its empirical distribution 
from a data sample (a training dataset or mini-batch thereof) ${\cal D} = \{ (x_i, y_i) \}_{i=1}^m$. The objective function in the double minimization \eqref{eq:th2-5} then becomes 
\begin{align}
J_{\cal D} (\boldsymbol{\theta}, \{Q_{y, i}\}_{y\in [C], i \in [N]}) & = \frac{1}{|{\cal D}|} \sum_{(x, y) \in {\cal D}} \mathsf{H}(y, q_{x,  \boldsymbol{\theta}}) + \nonumber \\
& \frac{\lambda}{\omega} \sum_y \frac{|{\cal D}_y|}{|{\cal D}|}
 \ln \left [ \frac{1}{N} \sum_{i=1}^N   \exp{ \{ \frac{\omega}{|{\cal D}_y|} \sum_{j \in {\cal D}_y} \mathrm{KL}\left( q_{X_j,  \boldsymbol{\theta}}^{\alpha_i }, Q_{y,i } \right)  \} }    \right  ]. \label{eq:joint_opt_max}
\end{align}

\subsection{Algorithm for Solving the Optimization in \eqref{eq:th2-5}} \label{subsec:alg}

The double minimization optimization problem in \eqref{eq:th2-5} naturally lends us an alternating algorithm that optimizes $\boldsymbol{\theta}$ and $\{Q_{y,i}\}_{y \in [C], i\in[N]}$ alternatively to minimize the objective function in \eqref{eq:th2-5} or \Cref{eq:joint_opt_max}, given the other is fixed.

Given $\{Q_{y,i}\}_{y \in [C], i\in[N]}$, $\boldsymbol{\theta}$  can be updated using the same first-order optimization method as in conventional deep learning, such as stochastic gradient descent applied over mini-batches.

Following \citet{10900607}, given $\boldsymbol{\theta}$, for each class $y$, $\{Q_{y,i }\}_{ i\in[N]}$ can be updated according to \eqref{eq:th2-4+} in the following manner: (1) we randomly sample a mini-batch of samples $|\mathfrak{B}_y|$ instances from the training set with ground truth label $y$; (2) $\{Q_{y,i}\}_{ i\in[N]}$ can be updated as 
\begin{align}\label{eq:empQ}
   & Q_{y,i } = \frac{\sum_{x \in \mathfrak{B}_y} q^{\alpha_i}_{x,\boldsymbol{\theta} }}{|\mathfrak{B}_y|} \quad  \forall i \in [N].
\end{align}

The proposed alternating algorithm for optimization problem \eqref{eq:th2-5} is summarized in \Cref{alg} \footnote{If the impact of the random mini-batch sampling and stochastic gradient descent is ignored, the alternating algorithm is guaranteed to converge in theory since given $\boldsymbol{\theta}  $, the optimal $\{Q_{y,i}\}_{y \in [C], i\in[N]}$ can be found analytically via \eqref{eq:empQ}, although it may not converge to a global minimum.}. To simplify our notation, we use  $(\cdot)^t_{b}$ to indicate parameters at the $b$-th batch updation during the $t$-th alternating iteration of the algorithm. We further write  $(\cdot)^t_{B}$ as $ (\cdot)^t$ whenever needed,  set $(\cdot)^t_{0}=(\cdot)^{t-1}$.

\begin{algorithm}[!t]
\caption{Conditional Mutual Information Minimized (CMIM) Method.} \label{alg}
\small
{\bfseries Input:} Training set $\mathcal{T}$, mini-batches $\{\mathcal{B}_b\}_{b \in [B]}$, number of epochs $\mathit{T}$, $\lambda$, $\beta$, $\omega$, $N$\\

{\bfseries Initialization:} Initialize  $\boldsymbol{\theta}^0$ and ${Q_{y,i}^0}_{\ y\in [C], i\in[N]}$.

\For{$t=1$ to $\mathit{T}$}
{
    \colorbox{pink}{[Sampling $\alpha_{i}$]} Randomly select N samples $\{\alpha_{i}\}_{i \in [N]}$ from interval $[0,\beta]$. \\
    \For{$b=1$ to $B$}
    {
    \colorbox{pink}{[Updating  $Q_{y,i}$]} For each class $y$, construct mini-batch $\{\mathcal{B}_y\}_{y\in[C]}$. Update $Q^{t}_{y,i}$, $\forall y \in[C];\ \forall i \in[N]$, according to \Cref{eq:empQ}.\\

    \colorbox{pink}{[Updating $\boldsymbol{\theta}$]}
    Fix ${Q_{y,i}^{t}}_{\ y\in [C], i\in[N]}$. Update $\boldsymbol{\theta}^t_{b-1}$ to $\boldsymbol{\theta}^t_{b}$ by stochastic gradient descent over the objective function \ref{eq:joint_opt_max}.   
    }
}
\textbf{Output:} Global model $\boldsymbol{\theta}^{\mathit{T}}$. 
\end{algorithm}

\begin{remark}
The compactness of output clusters alone is insufficient to ensure an undistillable DNN. Undistillability is a significantly stronger property. For instance, LS can improve the compactness of the feature space of a DNN, which may make the output probability of the DNN more compact \citep{muller2019does}. However, this compactness is not enough to make the DNN undistillable (see \cref{sec:exp})

% of output clusters to some extent \citep{muller2019does}. However, this improvement is not enough to make LS-trained models undistillable.

% According to \cref{def:Distillability}, a DNN is no longer undistillable if there exists any KD method that produces a distilled student outperforming the LS student. Our experiments in \cref{sec:exp} (\cref{Tab:BenchmarkCIFAR100}) show that for DNNs trained with LS, there is always at least one KD method that renders them distillable. In other words, LS is not strong enough to produce undistillable DNNs—LS-trained models can still be leveraged as teachers by competing methods to produce superior distilled students.

% To achieve an undistillable DNN, one must counteract \textit{all} possible KD methods and examine the compactness of each cluster under \textit{all} potential cluster-based power transforms. In particular, an undistillable DNN is analogous to an unbreakable encryption algorithm. Just as an unbreakable encryption scheme must withstand all possible attack methods, an undistillable DNN must resist distillation from all conceivable KD techniques.    
\end{remark}
\section{Experiments} \label{sec:exp}
\begin{table}[!ht]
% \vspace{-0.4in}
\caption{Top-1 accuracy (\%) of the knockoff student on CIFAR-100, TinyImageNet and ImageNet dataset (the results for CIFAR-100 and TinyImageNet are averaged over 3 runs). Green upward arrows ($\color{green}{\uparrow}$) and red downward arrows ($\color{red}{\downarrow}$) indicate whether the knockoff student was able to render the underlying DNN distillable.}
% \vspace{-0.15in}
\label{Tab:BenchmarkCIFAR100}
\begin{center}
\resizebox{1\textwidth}{!}{
\begin{tabular}{|cccccccccccc|}
\hline
\rowcolor{mygray} \multicolumn{12}{|c|}{CIFAR-100} \\ \toprule\hline
% \multicolumn{11}{|c|}{Post-Process Defence} \\ \hline
\multicolumn{1}{|c|}{Defense}                  & \multicolumn{1}{c|}{Model}                  & \multicolumn{1}{c|}{K-student} & LS & KD    & MKD   & DKD   & DIST  & HTC   & AVG   & \multicolumn{1}{c|}{Knockoff} & \multicolumn{1}{l|}{\cellcolor{pink}Best} \\ \hline
\multicolumn{1}{|c|}{\multirow{4}{*}{MAD}}    & \multicolumn{1}{c|}{\multirow{2}{*}{VGG16}} & \multicolumn{1}{c|}{VGG11}   & 71.94 & 68.55 $\color{red}{\downarrow}$ & 72.08 $\color{green}{\uparrow}$ & 53.32 $\color{red}{\downarrow}$ & 69.21 $\color{red}{\downarrow}$ & 71.19 $\color{red}{\downarrow}$ & 70.03 $\color{red}{\downarrow}$ & \multicolumn{1}{c|}{61.44 $\color{red}{\downarrow}$} & 72.08 $\color{green}{\uparrow}$ \\ \cline{3-3}
\multicolumn{1}{|c|}{}                        & \multicolumn{1}{c|}{}                       & \multicolumn{1}{c|}{SNV2}    & 72.65 & 72.50 $\color{red}{\downarrow}$  & 72.46 $\color{red}{\downarrow}$ & 7.64 \ \ $\color{red}{\downarrow}$ & 69.91 $\color{red}{\downarrow}$ & 71.37 $\color{red}{\downarrow}$ & 72.86 $\color{green}{\uparrow}$ & \multicolumn{1}{c|}{70.87 $\color{red}{\downarrow}$} & 72.86 $\color{green}{\uparrow}$ \\ \cline{2-11} 
\multicolumn{1}{|c|}{}                        & \multicolumn{1}{c|}{\multirow{2}{*}{RN50}}  & \multicolumn{1}{c|}{VGG11}   & 71.94 & 72.00 $\color{green}{\uparrow}$ & 72.04 $\color{green}{\uparrow}$ & 54.29 $\color{red}{\downarrow}$ & 71.57 $\color{red}{\downarrow}$ & 70.76 $\color{red}{\downarrow}$ & 70.73 $\color{red}{\downarrow}$ & \multicolumn{1}{c|}{61.73 $\color{red}{\downarrow}$ }  & 72.04 $\color{green}{\uparrow}$ \\ \cline{3-3}
\multicolumn{1}{|c|}{}                        & \multicolumn{1}{c|}{}                       & \multicolumn{1}{c|}{RN18}    & 78.76 & 77.76 $\color{red}{\downarrow}$ & 78.79 $\color{green}{\uparrow}$ & 43.73 $\color{red}{\downarrow}$ & 73.76 $\color{red}{\downarrow}$ & 77.89 $\color{red}{\downarrow}$ & 78.61 $\color{red}{\downarrow}$ & \multicolumn{1}{c|}{73.92 $\color{red}{\downarrow}$ } & 78.79 $\color{green}{\uparrow}$ \\ \hline
\multicolumn{1}{|c|}{\multirow{4}{*}{APGP}}   & \multicolumn{1}{c|}{\multirow{2}{*}{VGG16}} & \multicolumn{1}{c|}{VGG11}   & 71.94 & 71.92 $\color{green}{\uparrow}$ & 72.27 $\color{green}{\uparrow}$ & 27.24 $\color{red}{\downarrow}$ & 69.25 $\color{red}{\downarrow}$ & 70.08 $\color{red}{\downarrow}$ & 72.01 $\color{red}{\downarrow}$ & \multicolumn{1}{c|}{45.98 $\color{red}{\downarrow}$ } & 72.27 $\color{green}{\uparrow}$ \\ \cline{3-3}
\multicolumn{1}{|c|}{}                        & \multicolumn{1}{c|}{}                       & \multicolumn{1}{c|}{SNV2}    & 72.65 & 73.10  $\color{green}{\uparrow}$ & 73.75  $\color{green}{\uparrow}$ & 12.52 $\color{red}{\downarrow}$ & 71.04 $\color{red}{\downarrow}$ & 71.66 $\color{red}{\downarrow}$ & 73.20 $\color{green}{\uparrow}$ & \multicolumn{1}{c|}{9.48 \ \ $\color{red}{\downarrow}$ }  & 73.75 $\color{green}{\uparrow}$ \\ \cline{2-11} 
\multicolumn{1}{|c|}{}                        & \multicolumn{1}{c|}{\multirow{2}{*}{RN50}}  & \multicolumn{1}{c|}{VGG11}   & 71.94 & 71.91 $\color{red}{\downarrow}$ & 72.11 $\color{green}{\uparrow}$ & 9.74 \ \ $\color{red}{\downarrow}$  & 69.48 $\color{red}{\downarrow}$ & 71.36 $\color{red}{\downarrow}$ & 71.92 $\color{green}{\uparrow}$ & \multicolumn{1}{c|}{34.71 $\color{red}{\downarrow}$} & 72.11 $\color{green}{\uparrow}$ \\ \cline{3-3}
\multicolumn{1}{|c|}{}                        & \multicolumn{1}{c|}{}                       & \multicolumn{1}{c|}{RN18}    & 78.76 & 78.04 $\color{red}{\downarrow}$ & 79.06 $\color{green}{\uparrow}$ & 62.71 $\color{red}{\downarrow}$ & 77.32 $\color{red}{\downarrow}$ & 77.82 $\color{red}{\downarrow}$ & 77.90 $\color{red}{\downarrow}$ & \multicolumn{1}{c|}{2.57 \ \ $\color{red}{\downarrow}$}  & 79.06 $\color{green}{\uparrow}$ \\ \hline

\multicolumn{1}{|c|}{\multirow{4}{*}{RSP}}    & \multicolumn{1}{c|}{\multirow{2}{*}{VGG16}} & \multicolumn{1}{c|}{VGG11}   & 71.94 & 71.42 $\color{red}{\downarrow}$ & 72.04 $\color{green}{\uparrow}$ & 70.22 $\color{red}{\downarrow}$ & 70.80 $\color{red}{\downarrow}$ & 70.40 $\color{red}{\downarrow}$ & 71.56 $\color{red}{\downarrow}$ & \multicolumn{1}{c|}{31.04 $\color{red}{\downarrow}$} & 72.04 $\color{green}{\uparrow}$ \\ \cline{3-3}
\multicolumn{1}{|c|}{}                        & \multicolumn{1}{c|}{}                       & \multicolumn{1}{c|}{SNV2}    & 72.65 & 73.55 $\color{green}{\uparrow}$ & 72.95 $\color{green}{\uparrow}$ & 67.45 $\color{red}{\downarrow}$ & 72.19 $\color{red}{\downarrow}$ & 71.46 $\color{red}{\downarrow}$ & 72.27 $\color{red}{\downarrow}$ & \multicolumn{1}{c|}{26.09 $\color{red}{\downarrow}$} & 73.55 $\color{green}{\uparrow}$ \\ \cline{2-11} 
\multicolumn{1}{|c|}{}                        & \multicolumn{1}{c|}{\multirow{2}{*}{RN50}}  & \multicolumn{1}{c|}{VGG11}   & 71.94 & 71.97 $\color{green}{\uparrow}$ & 72.01 $\color{green}{\uparrow}$ & 69.53 $\color{red}{\downarrow}$ & 72.18 $\color{green}{\uparrow}$ & 70.87 $\color{red}{\downarrow}$ & 70.85 $\color{red}{\downarrow}$ & \multicolumn{1}{c|}{46.68 $\color{red}{\downarrow}$ } & 72.18 $\color{green}{\uparrow}$ \\ \cline{3-3}
\multicolumn{1}{|c|}{}                        & \multicolumn{1}{c|}{}                       & \multicolumn{1}{c|}{RN18}    & 78.76 & 77.78 $\color{red}{\downarrow}$ & 77.79 $\color{red}{\downarrow}$ & 77.01 $\color{red}{\downarrow}$ & 78.88 $\color{green}{\uparrow}$ & 78.00 $\color{red}{\downarrow}$ & 78.13 $\color{red}{\downarrow}$ & \multicolumn{1}{c|}{55.86 $\color{red}{\downarrow}$} & 78.88 $\color{green}{\uparrow}$ \\ \hline
% \multicolumn{11}{|c|}{Model Stealing Resistant Training} \\ \hline

\multicolumn{1}{|c|}{\multirow{4}{*}{NT}}     & \multicolumn{1}{c|}{\multirow{2}{*}{VGG16}} & \multicolumn{1}{c|}{VGG11}   & 71.94 & 71.40 $\color{red}{\downarrow}$ & 73.44 $\color{green}{\uparrow}$ & 71.47 $\color{red}{\downarrow}$ & 71.33 $\color{red}{\downarrow}$ & 70.77 $\color{red}{\downarrow}$ & 71.58 $\color{red}{\downarrow}$ & \multicolumn{1}{c|}{63.56 $\color{red}{\downarrow}$} & 73.44 $\color{green}{\uparrow}$ \\ \cline{3-3}
\multicolumn{1}{|c|}{}                        & \multicolumn{1}{c|}{}                       & \multicolumn{1}{c|}{SNV2}    & 72.65 & 72.44 $\color{red}{\downarrow}$ & 72.70 $\color{green}{\uparrow}$ & 6.24 \ \ $\color{red}{\downarrow}$ & 72.04 $\color{red}{\downarrow}$ & 70.75 $\color{red}{\downarrow}$ & 72.83 $\color{green}{\uparrow}$ & \multicolumn{1}{c|}{6.32 \ \ $\color{red}{\downarrow}$}  & 72.83 $\color{green}{\uparrow}$ \\ \cline{2-11} 
\multicolumn{1}{|c|}{}                        & \multicolumn{1}{c|}{\multirow{2}{*}{RN50}}  & \multicolumn{1}{c|}{VGG11}   & 71.94 & 72.01 $\color{green}{\uparrow}$ & 72.03 $\color{green}{\uparrow}$ & 71.55 $\color{red}{\downarrow}$ & 71.88 $\color{red}{\downarrow}$ & 70.16 $\color{red}{\downarrow}$ & 71.94 $\color{red}{\downarrow}$ & \multicolumn{1}{c|}{62.94 $\color{red}{\downarrow}$} & 72.03 $\color{green}{\uparrow}$ \\ \cline{3-3}
\multicolumn{1}{|c|}{}                        & \multicolumn{1}{c|}{}                       & \multicolumn{1}{c|}{RN18}    & 78.76 & 78.41 $\color{red}{\downarrow}$ & 78.92 $\color{green}{\uparrow}$ & 79.26 $\color{green}{\uparrow}$ & 78.99 $\color{green}{\uparrow}$ & 77.94 $\color{red}{\downarrow}$ & 78.33 $\color{red}{\downarrow}$ & \multicolumn{1}{c|}{68.96 $\color{red}{\downarrow}$ } & 79.26 $\color{green}{\uparrow}$ \\ \hline

\multicolumn{1}{|c|}{\multirow{4}{*}{SNT}}    & \multicolumn{1}{c|}{\multirow{2}{*}{VGG16}} & \multicolumn{1}{c|}{VGG11}   & 71.94 & 72.06 $\color{green}{\uparrow}$ & 72.28 $\color{green}{\uparrow}$ & 4.92 \ \  $\color{red}{\downarrow}$  & 71.98 $\color{green}{\uparrow}$ & 70.60 $\color{red}{\downarrow}$ & 71.63 $\color{red}{\downarrow}$ & \multicolumn{1}{c|}{64.08 $\color{red}{\downarrow}$} & 72.06 $\color{green}{\uparrow}$ \\ \cline{3-3}
\multicolumn{1}{|c|}{}                        & \multicolumn{1}{c|}{}                       & \multicolumn{1}{c|}{SNV2}    & 72.65 & 72.94 $\color{green}{\uparrow}$ & 73.17 $\color{green}{\uparrow}$ & 72.78 $\color{green}{\uparrow}$ & 72.22 $\color{red}{\downarrow}$ & 71.22 $\color{red}{\downarrow}$ & 72.74 $\color{green}{\uparrow}$ & \multicolumn{1}{c|}{6.22 \ \  $\color{red}{\downarrow}$}  & 73.17 $\color{green}{\uparrow}$ \\ \cline{2-11} 
\multicolumn{1}{|c|}{}                        & \multicolumn{1}{c|}{\multirow{2}{*}{RN50}}  & \multicolumn{1}{c|}{VGG11}   & 71.94 & 72.02 $\color{green}{\uparrow}$ & 72.12 $\color{green}{\uparrow}$ & 72.32 $\color{green}{\uparrow}$ & 71.70 $\color{red}{\downarrow}$ & 70.66 $\color{red}{\downarrow}$ & 71.65 $\color{red}{\downarrow}$ & \multicolumn{1}{c|}{62.94 $\color{red}{\downarrow}$ } & 72.32 $\color{green}{\uparrow}$ \\ \cline{3-3}
\multicolumn{1}{|c|}{}                        & \multicolumn{1}{c|}{}                       & \multicolumn{1}{c|}{RN18}    & 78.76 & 78.25 $\color{red}{\downarrow}$ & 78.48 $\color{red}{\downarrow}$ & 78.82 $\color{green}{\uparrow}$ & 78.14 $\color{red}{\downarrow}$ & 78.45 $\color{red}{\downarrow}$ & 78.38 $\color{red}{\downarrow}$ & \multicolumn{1}{c|}{67.71 $\color{red}{\downarrow}$ } & 78.82  $\color{green}{\uparrow}$\\ \hline

\multicolumn{1}{|c|}{\multirow{4}{*}{ST}}    & \multicolumn{1}{c|}{\multirow{2}{*}{VGG16}} & \multicolumn{1}{c|}{VGG11}   & 71.94 & 72.09 $\color{green}{\uparrow}$ & 72.01 $\color{green}{\uparrow}$ & 71.63  $\color{red}{\downarrow}$  & 71.93 $\color{red}{\downarrow}$ & 71.16 $\color{red}{\downarrow}$ & 71.63 $\color{red}{\downarrow}$ & \multicolumn{1}{c|}{63.32 $\color{red}{\downarrow}$} & 72.09 $\color{green}{\uparrow}$ \\ \cline{3-3}
\multicolumn{1}{|c|}{}                        & \multicolumn{1}{c|}{}                       & \multicolumn{1}{c|}{SNV2}    & 72.65 & 72.64 $\color{red}{\downarrow}$ & 72.67 $\color{green}{\uparrow}$ & 70.53 $\color{red}{\downarrow}$ & 72.24 $\color{red}{\downarrow}$ & 71.32 $\color{red}{\downarrow}$ & 72.42 $\color{red}{\downarrow}$ & \multicolumn{1}{c|}{69.46 $\color{red}{\downarrow}$}  & 72.67 $\color{green}{\uparrow}$ \\ \cline{2-11} 
\multicolumn{1}{|c|}{}                        & \multicolumn{1}{c|}{\multirow{2}{*}{RN50}}  & \multicolumn{1}{c|}{VGG11}   & 71.94 & 72.00 $\color{green}{\uparrow}$ & 72.13 $\color{green}{\uparrow}$ & 71.62 $\color{red}{\downarrow}$ & 71.76 $\color{red}{\downarrow}$ & 70.54 $\color{red}{\downarrow}$ & 71.73 $\color{red}{\downarrow}$ & \multicolumn{1}{c|}{65.43 $\color{red}{\downarrow}$ } & 72.13 $\color{green}{\uparrow}$ \\ \cline{3-3}
\multicolumn{1}{|c|}{}                        & \multicolumn{1}{c|}{}                       & \multicolumn{1}{c|}{RN18}    & 78.76 & 78.96 $\color{green}{\uparrow}$ & 79.02 $\color{green}{\uparrow}$ & 78.35 $\color{red}{\downarrow}$ & 78.31 $\color{red}{\downarrow}$ & 78.36 $\color{red}{\downarrow}$ & 78.81 $\color{green}{\uparrow}$ & \multicolumn{1}{c|}{72.87 $\color{red}{\downarrow}$ } & 79.02  $\color{green}{\uparrow}$\\ \hline

\multicolumn{1}{|c|}{\multirow{4}{*}{LS}}     & \multicolumn{1}{c|}{\multirow{2}{*}{VGG16}} & \multicolumn{1}{c|}{VGG11}   & 71.94 & 71.90 $\color{red}{\downarrow}$ & 72.00 $\color{green}{\uparrow}$ & 71.57 $\color{red}{\downarrow}$ & 70.89 $\color{red}{\downarrow}$ & 70.66 $\color{red}{\downarrow}$ & 71.76 $\color{red}{\downarrow}$ & \multicolumn{1}{c|}{63.49 $\color{red}{\downarrow}$ } & 72.00 $\color{green}{\uparrow}$ \\ \cline{3-3}
\multicolumn{1}{|c|}{}                        & \multicolumn{1}{c|}{}                       & \multicolumn{1}{c|}{SNV2}    & 72.65 & 72.87 $\color{green}{\uparrow}$ & 73.52 $\color{green}{\uparrow}$ & 70.01 $\color{red}{\downarrow}$ & 71.49 $\color{red}{\downarrow}$ & 71.70 $\color{red}{\downarrow}$ & 73.01 $\color{green}{\uparrow}$ & \multicolumn{1}{c|}{65.20 $\color{red}{\downarrow}$} & 73.52 $\color{green}{\uparrow}$ \\ \cline{2-11} 
\multicolumn{1}{|c|}{}                        & \multicolumn{1}{c|}{\multirow{2}{*}{RN50}}  & \multicolumn{1}{c|}{VGG11}   & 71.94 & 71.82 $\color{red}{\downarrow}$ & 71.99 $\color{green}{\uparrow}$ & 71.95 $\color{red}{\downarrow}$ & 70.77 $\color{red}{\downarrow}$ & 70.86 $\color{red}{\downarrow}$ & 71.88 $\color{red}{\downarrow}$ & \multicolumn{1}{c|}{62.29 $\color{red}{\downarrow}$ } & 71.99 $\color{green}{\uparrow}$ \\ \cline{3-3}
\multicolumn{1}{|c|}{}                        & \multicolumn{1}{c|}{}                       & \multicolumn{1}{c|}{RN18}    & 78.76 & 77.72 $\color{red}{\downarrow}$ & 77.82 $\color{red}{\downarrow}$ & 79.37 $\color{green}{\uparrow}$ & 78.33 $\color{red}{\downarrow}$ & 78.31 $\color{red}{\downarrow}$ & 77.91 $\color{red}{\downarrow}$ & \multicolumn{1}{c|}{63.36 $\color{red}{\downarrow}$ } & 79.37 $\color{green}{\uparrow}$ \\ \hline

\multicolumn{1}{|c|}{\multirow{4}{*}{CMIM}} & \multicolumn{1}{c|}{\multirow{2}{*}{VGG16}} & \multicolumn{1}{c|}{VGG11}   & 71.94 & 71.87 $\color{red}{\downarrow}$ & 71.64 $\color{red}{\downarrow}$ & 71.56 $\color{red}{\downarrow}$ & 70.34 $\color{red}{\downarrow}$ & 71.71 $\color{red}{\downarrow}$ & 71.42 $\color{red}{\downarrow}$ & \multicolumn{1}{c|}{66.89 $\color{red}{\downarrow}$ } & 71.87 $\color{red}{\downarrow}$ \\ \cline{3-3}
\multicolumn{1}{|c|}{}                        & \multicolumn{1}{c|}{}                       & \multicolumn{1}{c|}{SNV2}    & 72.65 & 72.53 $\color{red}{\downarrow}$ & 71.44 $\color{red}{\downarrow}$ & 72.46 $\color{red}{\downarrow}$ & 71.45 $\color{red}{\downarrow}$ & 71.59 $\color{red}{\downarrow}$ & 71.94 $\color{red}{\downarrow}$ & \multicolumn{1}{c|}{64.45 $\color{red}{\downarrow}$} &  72.53 $\color{red}{\downarrow}$ \\ \cline{2-11} 
\multicolumn{1}{|c|}{}                        & \multicolumn{1}{c|}{\multirow{2}{*}{RN50}}  & \multicolumn{1}{c|}{VGG11}   & 71.94 & 71.54 $\color{red}{\downarrow}$ & 71.34 $\color{red}{\downarrow}$ & 71.77 $\color{red}{\downarrow}$ & 71.86 $\color{red}{\downarrow}$ & 69.32 $\color{red}{\downarrow}$ & 71.70 $\color{red}{\downarrow}$ & \multicolumn{1}{c|}{60.58 $\color{red}{\downarrow}$} & 71.86 $\color{red}{\downarrow}$ \\ \cline{3-3}
\multicolumn{1}{|c|}{}                        & \multicolumn{1}{c|}{}                       & \multicolumn{1}{c|}{RN18}    & 78.76 & 78.21 $\color{red}{\downarrow}$ & 78.16 $\color{red}{\downarrow}$ & 78.13 $\color{red}{\downarrow}$ & 77.56 $\color{red}{\downarrow}$ & 77.23 $\color{red}{\downarrow}$ & 78.64 $\color{red}{\downarrow}$ & \multicolumn{1}{c|}{65.88 $\color{red}{\downarrow}$} & 78.64 $\color{red}{\downarrow}$ \\  \toprule\hline

\rowcolor{mygray} \multicolumn{12}{|c|}{TinyImageNet} \\ \toprule\hline 

\multicolumn{1}{|c}{\multirow{2}{*}{RSP}} & \multicolumn{1}{|c|}{RN34} & \multicolumn{1}{|c|}{RN18}   & 63.56 & 63.54 $\color{red}{\downarrow}$ & 64.32 $\color{green}{\uparrow}$ & 64.01 $\color{green}{\uparrow}$ & 63.27 $\color{red}{\downarrow}$ & 63.54 $\color{red}{\downarrow}$ & 62.15 $\color{red}{\downarrow}$ & \multicolumn{1}{c|}{55.43 $\color{red}{\downarrow}$ } & 64.32 $\color{green}{\uparrow}$ \\ \cline{2-11}
 & \multicolumn{1}{|c|}{RN50}  & \multicolumn{1}{|c|}{SNV2}   & 60.61 & 60.18 $\color{red}{\downarrow}$ & 60.76  $\color{green}{\uparrow}$  & 56.26 $\color{red}{\downarrow}$ & 56.43 $\color{red}{\downarrow}$ & 60.96  $\color{green}{\uparrow}$  & 60.15 $\color{red}{\downarrow}$ & \multicolumn{1}{c|}{54.01 $\color{red}{\downarrow}$} & 60.96  $\color{green}{\uparrow}$  \\ \cline{2-11} \hline

 \multicolumn{1}{|c}{\multirow{2}{*}{ST}} & \multicolumn{1}{|c|}{RN34} & \multicolumn{1}{|c|}{RN18}   & 63.56 & 63.96  $\color{green}{\uparrow}$ & 64.12  $\color{green}{\uparrow}$  & 63.25 $\color{red}{\downarrow}$ & 63.51 $\color{red}{\downarrow}$ & 63.49 $\color{red}{\downarrow}$ & 63.84  $\color{green}{\uparrow}$  & \multicolumn{1}{c|}{57.42 $\color{red}{\downarrow}$ } & 64.12  $\color{green}{\uparrow}$  \\ \cline{2-11}
 & \multicolumn{1}{|c|}{RN50}  & \multicolumn{1}{|c|}{SNV2}   & 60.61 & 61.23  $\color{green}{\uparrow}$ & 61.36  $\color{green}{\uparrow}$  & 60.43 $\color{red}{\downarrow}$ & 60.32 $\color{red}{\downarrow}$ & 60.22 $\color{red}{\downarrow}$ & 61.13  $\color{green}{\uparrow}$  & \multicolumn{1}{c|}{55.84 $\color{red}{\downarrow}$} & 61.36  $\color{green}{\uparrow}$ \\ \cline{2-11} \hline

\multicolumn{1}{|c}{\multirow{2}{*}{NT}} & \multicolumn{1}{|c|}{RN34} & \multicolumn{1}{|c|}{RN18}   & 63.56 & 63.27 $\color{red}{\downarrow}$ & 64.49 $\color{green}{\uparrow}$  & 64.67  $\color{green}{\uparrow}$  & 63.43 $\color{red}{\downarrow}$ & 63.50 $\color{red}{\downarrow}$ & 64.43  $\color{green}{\uparrow}$ & \multicolumn{1}{c|}{53.11 $\color{red}{\downarrow}$ } & 64.67 $\color{green}{\uparrow}$ \\ \cline{2-11}
 & \multicolumn{1}{|c|}{RN50}  & \multicolumn{1}{|c|}{SNV2}   & 60.61 & 59.57 $\color{red}{\downarrow}$ & 61.55  $\color{green}{\uparrow}$ & 31.55 $\color{red}{\downarrow}$ & 60.03 $\color{red}{\downarrow}$ & 60.98  $\color{green}{\uparrow}$ & 60.31 $\color{red}{\downarrow}$ & \multicolumn{1}{c|}{50.94 $\color{red}{\downarrow}$} & 61.55  $\color{green}{\uparrow}$  \\ \cline{2-11} \hline

\multicolumn{1}{|c}{\multirow{2}{*}{LS}} & \multicolumn{1}{|c|}{RN34} & \multicolumn{1}{|c|}{RN18}   & 63.56 & 63.74  $\color{green}{\uparrow}$  & 64.01  $\color{green}{\uparrow}$  & 64.23  $\color{green}{\uparrow}$ & 63.51 $\color{red}{\downarrow}$ & 64.20  $\color{green}{\uparrow}$  & 63.04$ \color{red}{\downarrow}$ & \multicolumn{1}{c|}{57.43 $\color{red}{\downarrow}$ } & 64.23 $\color{green}{\uparrow}$ \\ \cline{2-11}

 & \multicolumn{1}{|c|}{RN50}  & \multicolumn{1}{|c|}{SNV2}   & 60.61 & 60.32 $\color{red}{\downarrow}$ & 60.93  $\color{green}{\uparrow}$  & 60.74  $\color{green}{\uparrow}$ & 60.11 $\color{red}{\downarrow}$ & 60.46 $\color{red}{\downarrow}$ & 60.14 $\color{red}{\downarrow}$ & \multicolumn{1}{c|}{52.96 $\color{red}{\downarrow}$} & 60.93  $\color{green}{\uparrow}$ \\ \cline{2-11} \hline

 \multicolumn{1}{|c}{\multirow{2}{*}{CMIM}} & \multicolumn{1}{|c|}{RN34} & \multicolumn{1}{|c|}{RN18}   & 63.53 & 62.89 $\color{red}{\downarrow}$ & 63.15 $\color{red}{\downarrow}$ & 62.94 $\color{red}{\downarrow}$ & 63.28 $\color{red}{\downarrow}$ & 61.57 $\color{red}{\downarrow}$ & 62.96 $\color{red}{\downarrow}$ & \multicolumn{1}{c|}{56.13 $\color{red}{\downarrow}$ } & 63.28 $\color{red}{\downarrow}$ \\ \cline{2-11}
 & \multicolumn{1}{|c|}{RN50}  & \multicolumn{1}{|c|}{SNV2}   & 60.61 & 57.57 $\color{red}{\downarrow}$ & 59.32 $\color{red}{\downarrow}$ & 60.58 $\color{red}{\downarrow}$ & 59.41 $\color{red}{\downarrow}$ & 59.33 $\color{red}{\downarrow}$ & 60.42 $\color{red}{\downarrow}$ & \multicolumn{1}{c|}{56.91 $\color{red}{\downarrow}$} & 60.58 $\color{red}{\downarrow}$ \\ \cline{2-11} \toprule\hline
\rowcolor{mygray}\multicolumn{12}{|c|}{ImageNet} \\ \toprule \hline 
\multicolumn{1}{|c|}{\multirow{2}{*}{ST}}   & \multicolumn{1}{c|}{\multirow{2}{*}{RN34}} & \multicolumn{1}{c|}{RN18}      & 70.89 & 70.74 $\color{red}{\downarrow}$ & 71.02 $\color{green}{\uparrow}$ & 70.02 $\color{red}{\downarrow}$& 69.94 $\color{red}{\downarrow}$ & 70.91 $\color{green}{\uparrow}$ & 71.00 $\color{green}{\uparrow}$ & \multicolumn{1}{c|}{63.24 $\color{red}{\downarrow}$ }   & 71.02 $\color{green}{\uparrow}$ \\ \cline{3-3}
\multicolumn{1}{|c|}{}                      & \multicolumn{1}{c|}{}                      & \multicolumn{1}{c|}{MNV2}      & 70.93 & 71.03 $\color{green}{\uparrow}$ & 71.25 $\color{green}{\uparrow}$ & 69.32 $\color{red}{\downarrow}$ & 70.53 $\color{red}{\downarrow}$ & 70.69 $\color{red}{\downarrow}$ & 71.06 $\color{green}{\uparrow}$ & \multicolumn{1}{c|}{54.53 $\color{red}{\downarrow}$}  & 71.25 $\color{green}{\uparrow}$ \\ \hline
\multicolumn{1}{|c|}{\multirow{2}{*}{CMIM}} & \multicolumn{1}{c|}{\multirow{2}{*}{RN34}} & \multicolumn{1}{c|}{RN18}      & 70.89 & 70.44 $\color{red}{\downarrow}$ & 70.69 $\color{red}{\downarrow}$ & 69.97 $\color{red}{\downarrow}$ & 70.59 $\color{red}{\downarrow}$ & 70.63 $\color{red}{\downarrow}$ & 70.53 $\color{red}{\downarrow}$ & \multicolumn{1}{c|}{59.34 $\color{red}{\downarrow}$ }   & 70.69  $\color{red}{\downarrow}$ \\ \cline{3-3}
\multicolumn{1}{|c|}{}                      & \multicolumn{1}{c|}{}                      & \multicolumn{1}{c|}{MNV2}      & 70.93 & 70.21 $\color{red}{\downarrow}$ & 70.72 $\color{red}{\downarrow}$ & 69.97 $\color{red}{\downarrow}$ & 70.44 $\color{red}{\downarrow}$ & 70.86 $\color{red}{\downarrow}$ & 70.20 $\color{red}{\downarrow}$ & \multicolumn{1}{c|}{55.24 $\color{red}{\downarrow}$}  & 70.86 $\color{red}{\downarrow}$ \\ \hline
\end{tabular}}
% \vskip -0.6in
\end{center}
\end{table}

In this section, we demonstrate the effectiveness of CMIM by comparing it with several state-of-the-art alternatives. Specifically, we first report the accuracy that a knockoff student can achieve by deploying different logit-based KD (attack) methods in \cref{exp:acck}. In all the experiments, when testing the distillibality of the trained DNNs using the benchmark defense methods and CMIM, we compare the knockoff student's accuracy  (i) when it attempts to steal the IP of protected DNN using logit-based (attack) methods with (ii) when it trains its model using the LS. If the former outperforms the latter, we conclude that the knockoff makes the underlying DNN distillable. Next, in \Cref{exp:ACC}, we report the classification accuracy of the \textit{protected} models trained by the different defense methods. Lastly, in \Cref{sec:visual}, we visualize the output cluster of models trained by CMIM, CE and NT.

\subsection{Knockoff Student Accuracy}\label{exp:acck}

\noindent $\bullet$ \textbf{Datasets:} We conduct extensive experiments on three image classification dataset, namely CIFAR-100 \citep{CIFAR100} TinyImageNet \citep{tinyiamgenet} and ImageNet \citep{imagenet}. For description of each dataset, please refer to \cref{sec:dataset}. 

\noindent $\bullet$ \textbf{Models:} To show the effectiveness of CMIM, we use different model architectural families for teacher and knockoff student models. 
To this end, we pick models from VGG family \citep{VGG}, ResNet family \citep{ResNet} (shortened as RN), ShuffleNetV2 \citep{ShuffleNetV2}, shortened as SNV2, and Mobilenetv2 \citep{sandler2018mobilenetv2} shortened as MNV2. Particularly, we have conducted experiments on the following (teacher-student) pairs for each dataset: (i) for CIFAR-100, we use four pairs $\{$(VGG16-VGG11), (VGG16-SNV2), (RN50-VGG11), (RN50-RN18)$\}$; (ii) for TinyImageNet, we use two pairs $\{$(RN34-RN18), (RN50-SNV2)$\}$; and for ImageNet we use two pairs $\{$(RN34-RN18),(RN34-MNV2)$\}$. 
 
\noindent $\bullet$ \textbf{Defense benchmark methods:}
For comprehensive comparisons, we benchmark CMIM with seven recently published defense methods: MAD \citep{MAD}, APGP \citep{APGP}, RSP \citep{RSP}, ST \citep{StingyTeacher}, NT \citep{nasty}, SNT \citep{SNT}, and LS\footnote{Although LS is not a defense method per se, it is observed that the models trained by LS reduce the knockoff student's accuracy. We discuss the rationale behind this in \Cref{Appendix:LsDefense}.} \citep{muller2019does}. 

\noindent $\bullet$ \textbf{Logit-based KD (attack) methods:} We use three logit-based KD methods that are primarily designed for when the teacher-student models are in cooperating mode, namely KD \citep{hinton2015distilling}, DKD \citep{DKD}, DIST \citep{DIST}; and four logit-based KD attacks methods that a knockoff student can deploy to make the protected DNNs possibly distillable, namely MKD \citep{yang2024markov}, HTC \citep{HTC}, AVG \citep{keser2023averager}, Knockoff \citep{orekondy2019knockoff}. We report all the training setups, including all the hyper-parameters used for both defense and attack methods in \Cref{Appendix:DefenseSetup}.

\noindent $\bullet$ \textbf{Results:} The accuracy that a knockoff student can achieve using the various \textit{(defense-attack)} combinations is summarized in \Cref{Tab:BenchmarkCIFAR100}. For accuracy variances, please refer to \Cref{app:variance}. In the table, we use the notation ``K-student'' to denote a knockoff student. The numbers in the column titled ``Best'' represent the highest accuracy obtained for each respective row, indicating the best possible performance a knockoff student can achieve using any of the listed distillation methods. 

As observed in \Cref{Tab:BenchmarkCIFAR100}, regardless of whether the teacher-student architectures are the same or different, DNNs trained with CMIM remain undistillable across all distillation methods. This is in stark contrast to DNNs trained using other defense techniques, which can still be successfully distilled to a certain degree. The results indicate that prior defense strategies do not offer complete resistance against knockoff students, whereas CMIM effectively prevents distillation, making it significantly more robust in protecting model knowledge.

We also perform an ablation study on the hyperparameters of CMIM—namely $\beta$, $N$, and $\omega$—as detailed in \cref{sec:abl}. The key findings are as follows:
\begin{itemize}
\item \textbf{Effect of $\beta$:} The accuracy of the knockoff student drops sharply when $\beta \geq 1$, suggesting that large values may destabilize optimization or excessively penalize the CMI term. This highlights the importance of carefully tuning $\beta$ to effectively balance the cross-entropy and CMI objectives.

\item \textbf{Effect of the number of samples \(N\):} Increasing the number of power samples \(N\) leads to a monotonic decrease in the knockoff student’s accuracy. This indicates that moderate values of \(N\) are sufficient to capture the necessary diversity for robust CMI estimation.

\item \textbf{Effect of the power coefficient \(\omega\):} Setting \(\omega > 30\) can lead to numerical instability, resulting in NaNs due to excessive exponentiation. Interestingly, the worst knockoff student performance is observed at \(\omega = 20\) and \(\omega = 30\), suggesting that these settings best approximate the extreme concentration behavior of \(\omega \to \infty\), thereby enhancing undistillability.
\end{itemize}

Additionally, we investigate CMIC’s impact on model calibration in \cref{Appendix:Caliberation}.

\subsection{Accuracy of Protected Models}   \label{exp:ACC}
In this section, we report the top-1 accuracy of the \textit{protected} models in \Cref{Tab:BenchmarkCIFAR100} trained using the benchmark defense methods with those trained by CMIM. The results are summarized in \Cref{tab:acc,tab:imagenetacc}. As observed, the models trained by CMIM have the highest classification accuracy compared to the benchmark methods. This is because for the models trained by CMIM, the clusters corresponding to the output probability of the DNNs are very concentrated, facilitating easier classification of samples from different classes.

\begin{table}[h]
% \vskip -0.1in
\caption{Top-1 accuracy (\%) of models trained by defense methods on CIFAR-100 and TinyImageNet. The best and second best results are \textbf{bolded} and \underline{underlined}, respectively.}
% \vskip -0.1in
\resizebox{1\textwidth}{!}{\begin{tabular}{|cccccccccc|ccccccc|}
\hline
\rowcolor{mygray}\multicolumn{10}{|c|}{CIFAR100} & \multicolumn{7}{c|}{TinyImageNet} \\ \toprule \hline
\multicolumn{1}{|c|}{Model} & CE    & MAD   & APGP  & RSP   & \multicolumn{1}{l}{ST} & NT    & SNT   & LS    & CMIM  & \multicolumn{1}{c|}{Model} & CE    & RSP   & ST    & NT    & LS    & CMIM  \\ \hline
\multicolumn{1}{|c|}{VGG16}   & 73.75 & 73.75 & 73.84 & 73.71 & 73.75                  & 73.75 & 72.59 & \textbf{73.90} & \underline{73.84} & \multicolumn{1}{c|}{RN34}    & 65.39 & 65.21 & 65.39 & 65.23 & \underline{65.45} & \textbf{65.99} \\
\multicolumn{1}{|c|}{RN50}    & 77.81 & 77.81 & 77.56 & 77.63 & 77.81                  & 77.31 & 77.77 & \underline{78.45} & \textbf{78.72} & \multicolumn{1}{c|}{RN50}    & \underline{66.14} & 65.91 & 66.13 & 66.06 & 66.09 & \textbf{66.93} \\ \hline
\end{tabular}} \label{tab:acc}
\end{table}

\begin{table}[h!]
\centering
\caption{Top-1 accuracy (\%) of models trained by defense methods on ImageNet.}
% \vskip -0.1in
\resizebox{0.3\textwidth}{!}{\begin{tabular}{|cccc|}
\hline
\rowcolor{mygray} \multicolumn{4}{|c|}{ImageNet} \\ \toprule \hline 
\multicolumn{1}{|c|}{Model} & CE    & ST                        & CMIM  \\ \hline
\multicolumn{1}{|c|}{RN34}    & \underline{73.31} & \multicolumn{1}{c}{73.30} & \textbf{73.69} \\ \hline
\end{tabular}} \label{tab:imagenetacc}
\end{table}

The results in \cref{tab:acc} motivate us to test the top-1 accuracy of additional models trained by CMIM and compare them with those trained by CE loss (see \cref{sec:additional}).

It is worth noting that the primary focus of this paper is not on increasing the accuracy of DNNs but on developing a method to train undistillable DNNs. While many existing methods in the literature can enhance a DNN’s accuracy, they do not address the critical challenge of making DNNs undistillable.

Our approach is the first in the literature that effectively trains undistillable models robust against a wide range of existing KD methods. The improvement in accuracy observed in our results is a by-product of our method and not its primary goal. This improvement arises from the unique properties of our approach rather than replicating the effects of label smoothing or similar techniques.

\begin{figure*}[!t] 
% \vskip -0.2in
	\centering
	\subfloat[LS]{\includegraphics[width=0.30\columnwidth]{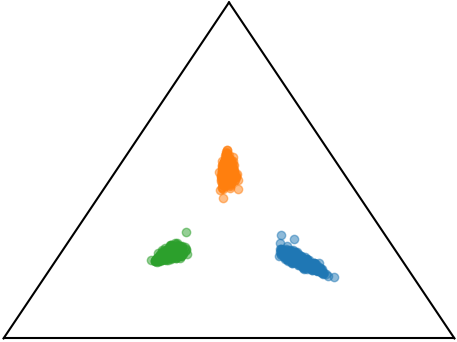}\label{fig:CESimplex}} 
	\subfloat[Nasty teacher] {\includegraphics[width=0.30\columnwidth]{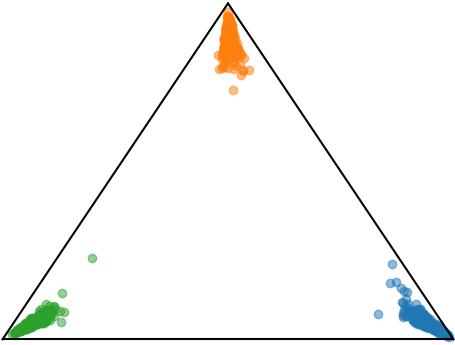}\label{fig:NastySimplex}}
 \subfloat[CMIM]{\includegraphics[width=0.30\columnwidth]{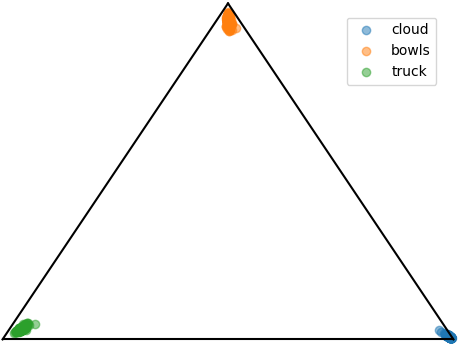}\label{fig:CMIMSimplex}}
	% \vskip -0.1in
    \caption{Visualization of three projected probability clusters for ResNet-50 trained on CIFAR-100 using (a) LS, (b) NT, and (c) CMIM.}
    \label{fig:Simplex}
	% \vskip -0.3in
\end{figure*}

\subsection{Visualizing the Output Clusters} \label{sec:visual}
In this subsection, we aim to visualize the output clusters for models trained using CE, NT, and CMIM. To achieve this, we follow the visualization approach introduced by \citet{10900607}. Specifically, we randomly select three labels from the CIFAR-100 dataset. For each probability distribution corresponding to these three labels, we extract only the probabilities associated with these selected labels and normalize them to form three-dimensional probability vectors. These vectors are then projected onto a two-dimensional simplex, allowing us to visualize the clustering behavior of each model. By applying this transformation, we obtain a clear representation of how the models distribute their probability mass across different output categories.  

The resulting simplexes for ResNet-50 models trained with LS, Nasty Teacher, and the CMIM framework are shown in \Cref{fig:Simplex} \footnote{In Appendix \ref{Appendix:MoreSimplex}, we present the visualization of three different projected probability clusters for ResNet-50 trained on CIFAR-100.}. To ensure a consistent comparison, we applied the same power transform $\alpha = 4$ for all visualizations. As observed, the clusters for the model trained with CMIM are highly concentrated near the corners of the simplex, closely resembling one-hot vectors. This indicates that the output distributions are highly compact, making it difficult for a knockoff student to surpass LS regularization when attempting to distill knowledge from a CMIM-trained model. This increased compactness plays a crucial role in enhancing the model's resistance to knowledge distillation.

Lastly, we clarify the distinction between ``highly concentrated output clusters'' and ``overly confident predictions''. A highly concentrated output cluster does not necessarily imply that the model produces overly confident predictions. This is because the clusters can be concentrated around points that are not close to one-hot labels (the corners of the probability simplex). As a result, the model can have concentrated outputs without being overly confident. These are two separate concepts. 

To illustrate this, we train an RN50 model on the CIFAR-100 dataset using three different methods: CMIM, CE, and LS. After training, we evaluate the model's confidence by measuring the average entropy of its output probability vectors on the test dataset. The results, presented in \cref{Tab:EntropyVals}, indicate that the entropy for CMIM is higher than that for CE, demonstrating that CMIM produces less confident output probabilities.

\begin{table}[h!]
\caption{Entropy Value of the Models Trained by CMIM, CE and LS}
\begin{center}
\label{Tab:EntropyVals}
% \vskip -0.2in
\resizebox{0.35\textwidth}{!}{\begin{tabular}{|c|c|c|c|c|}
\hline \rowcolor{mygray} 
Methods & CMIM & CE & LS    \\ \toprule  \hline
Entropy  & 0.102 & 0.064	& 0.152		 \\ \hline
\end{tabular}}
\end{center}
\end{table}

Additionally, our experiments in \cref{Tab:BenchmarkCIFAR100} further support the above-mentioned claim: the CMIM method generates more concentrated output clusters while also improving the model’s accuracy on the held-out test dataset compared to models trained with the conventional CE loss. This observation is consistent with the findings of \citet{10900607}, which demonstrate that training DNNs to produce highly concentrated output clusters can lead to improved test accuracy.

\subsection{Why Prior Defense Methods Can be Made Distillable?}
In this section, we address the question of why DNNs trained using all prior KD-resistance defense methods can still be made distillable, as demonstrated in \Cref{exp:acck}. The key reason behind this lies in the fundamental characteristics of these defense methods and their susceptibility to distillation under certain conditions. Specifically, by appropriately tuning the power transform parameter $\alpha$, the models trained using these defense methods can attain a relatively high CMI value in comparison to our proposed method, CMIM (as illustrated in \Cref{fig:distillable}). As though, compared to the CMIM model, all other defense methods provide more information for the student to leverage in improving their own performance. 
 
This suggests that, despite their initial resistance, the defense methods fail to enforce undistillability rigorously across all distillation settings. When the correct $\alpha$ value is selected during the distillation process, a logit-based KD approach can leverage this property to effectively distill knowledge from these supposedly resistant DNNs. Consequently, these models can still be exploited to produce distilled students that outperform the baseline LS student, demonstrating that prior defense strategies do not provide robust protection against all KD methods.
\begin{figure}[!t] 
% \vspace{-4mm}
	\centering
	\subfloat[ResNet-50]{\includegraphics[width=0.4\columnwidth]{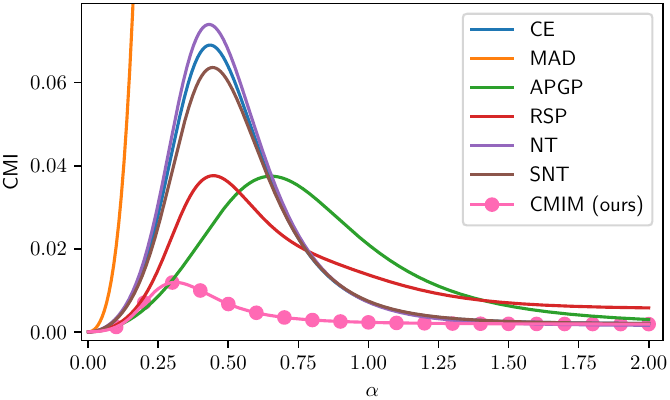}\label{fig:ResNet-50}}
	\subfloat[VGG16]{\includegraphics[width=0.4\columnwidth]{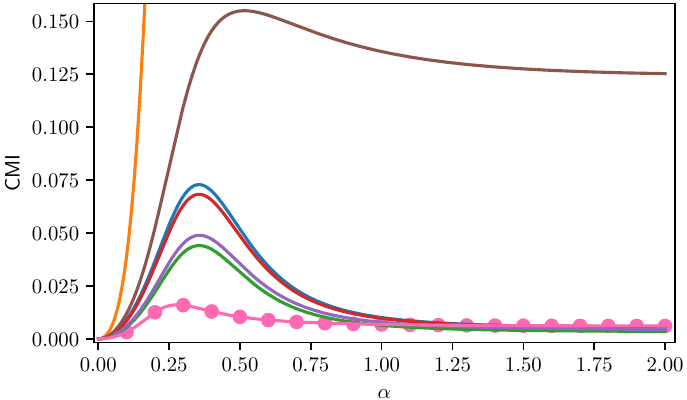}\label{fig:VGG16}}
	% \vskip -0.1in
    \caption{The CMI values for the models trained by different KD-resistance defence methods Vs. power transform value $\alpha$ for (a) ResNet-50, and (b) VGG16 trained on CIFAR-100 dataset. Compared to all other defense methods in the literature, CMIM effectively reduces peak CMI values under temperature scaling, which prevents the teacher model from being distilled by the knockoff student.}
    \label{fig:distillable}
	% \vskip -0.2in
\end{figure}

\section{Conclusion and Future Directions}
In this paper, from an information-theoretic perspective, we proposed a defence method against the threat posed by knockoff students attempting to steal the IP of pre-trained DNNs via logit-based KD methods. In particular, we proposed to minimize the CMI of the protected DNN across different power transform hyper-parameter values $\alpha$, while minimizing the conventional CE loss simultaneously. We referred to model trained by these framework as CMIM models. By conducting a series of experiments, we showed that, unlike the prior defense methods proposed in the literature, a knockoff student cannot render CMIM models distillable. In addition, we showed that the models trained by CMIM achieve higher classification accuracy compared to those trained by CE loss. 

Despite these promising results, our work has certain limitations. First, the evaluation of CMIM models is primarily empirical, as providing a formal theoretical proof of undistillability remains an open challenge.  Second, our approach introduces additional computational overhead compared to the conventional training using CE loss.

For future work, we aim to extend the CMIM method beyond standard classification tasks. Potential directions include adapting it for multi-label classification, regression problems, and safeguarding the intellectual property of cutting-edge models such as LLMs, CLIP, and diffusion models. Another promising direction is to enhance the model’s undistillability by adapting CMIM for integration with a contrastive learning framework, which regularizes the latent space rather than the output probability space. By broadening the scope of CMIM, we hope to further enhance its applicability and effectiveness across a wider range of machine learning paradigms.

\bibliography{tmlr}
\bibliographystyle{tmlr}

\newpage
\appendix
\section{Appendix}

\section{A Thorough Study of Related Works} \label{sec:app_rel}
\subsection{Logit-based KD methods} \label{sec:logit}
Knowledge Distillation (KD) has become a cornerstone technique for compressing large teacher models into smaller student models. This section reviews key logit-based KD methods and explores their advancements.

\citet{hinton2015distilling} introduced the foundational concept of KD, using KL divergence to match the student's softmax outputs to the teacher's. This work laid the groundwork for logit-based methods, as the softmax output directly relates to the logits. Later, \citet{moldovan2019path} proposed Path-KD, a method that utilizes the paths leading to the correct class in both teacher and student models for distillation. While not directly logit-based, it demonstrates the effectiveness of aligning decision-making processes. \citet{guo2021logit} proposed Logit-Like Distillation, addressing the capacity gap by matching the ranking of logits instead of their exact values [4]. This approach allows the student to learn the essential ordering of classes even with limited capacity. \citet{an2021relation}  proposed relation knowledge distillation (RKD), focusing on aligning relationships between class logits rather than individual values. This approach improves the student's ability to generalize to unseen data. \citet{DKD} introduced decoupled knowledge distillation (DKD), where it decouples the classical KD loss into two parts: target
class knowledge distillation and non-target class
knowledge distillation. \citet{huang2022knowledge} proposed DIST where they designed a KD method to distill better from a stronger teacher; indeed they claim that preserving the relations between the predictions of teacher and student would suffice for an effective KD. \citet{borup2021even} provided theoretical arguments for the importance of weighting the teacher outputs
with the ground-truth targets when performing self-distillation with kernel ridge regressions along
with a closed form solution for the optimal weighting parameter. \citet{hamidi2024train} propose to train the teacher model using mean square loss instead of the conventional cross entropy loss. \citet{salamah2025coded} proposed a KD framework that leverages adaptive JPEG encoding to make the  teacher's output responses less peaky for image classification tasks.

\subsection{Defense methods against Logit-based KD} 
As also discussed in \Cref{sec:related}, the defense methods against  the threat posed by knockoff students attempting to steal the IP of pre-trained DNNs via logit-based KD methods can be categorized into two approaches. Here, we elaborate on these two approaches.  

\noindent  \textbf{(I)} \textbf{Model stealing resistant training:} 
In this approach, DNNs are trained to reduce the accuracy of knockoff students while maintaining the original classification accuracy of the model. In particular, \citet{nasty} proposed a training algorithm named self-undermining KD to create nasty teachers (NT) that prevent knowledge leakage and unauthorized model stealing through KD, without compromising model accuracy. The nasty teacher is trained by minimizing the following objective function: 
\begin{align}
{\cal L}^{NT} = \mathsf{H}(y, q_x) - \epsilon ~ \mathrm{KL}(\tilde{q}_x, q_x),
\end{align}
where $\tilde{q}_x$ is a output of a pre-trained standard model.

Subsequently, \citet{SNT} proposed semantic nasty teachers (SNT) which improve the model stealing resistance of NT by disentangling semantic relationships in the output logits during teacher model training, which is crucial for successful KD.

\noindent \textbf{(II)} \textbf{post-training defence methods:}
The aim of these approaches is to deceive the knockoff by imposing minimal perturbations to the model's predictions. \citet{RSP} tested a variety of possible perturbation forms, and found that the reverse sigmoid perturbation (RSP) to be the most effective one. \citet{MAD} introduced maximizing angular deviation (MAD), a technique that perturbs the output probabilities, leading to an adversarial gradient signal that deviates significantly from the original gradient of the knockoff. To this end, they applied a randomly initialized model as the surrogate for the potential knockoff. More recently, \citet{APGP} proposed a plug-and-play generative perturbation model, dubbed as accuracy preserving generative perturbation (APGP), which can effectively defend KD-based model cloning, while preserve the model utility.

\subsection{Attack Methods Using Logit-based KD} \label{sec:attack}

\citet{HTC}  sought to circumvent the defense of nasty teachers and steal (or extract) its information. Specifically, they analyzed nasty teacher from two different angles and subsequently leverage them carefully to develop simple yet efficient methodologies, named as HTC and SCM, which enhance  learning from nasty teacher.

In AVG \citep{keser2023averager}, the authors noted that undistillable teachers exhibit multiple peaks in their softmax response, which are transferred to the student models. These peaks are considered to be the primary factor that misleads the student models. To mitigate the influence of the multiple peaks in the softmax response of teachers, they proposed transferring the mean of features with the same labels as the soft labels.

\citet{orekondy2019knockoff} introduced a technique called "Knockoff Nets" that allows an attacker to steal the functionality of black-box models. Remarkably, the attacker only needs to interact with the model by feeding it input data and observing the resulting predictions. By training a new model ("knockoff") on these input-prediction pairs, the attacker can create a copycat model that performs similarly to the original black box.

\section{Why LS reduce the knockoff student's accuracy?}\label{Appendix:LsDefense}
Original aiming to prevent overfitting and improve generalization, label smoothing was observed by \citet{muller2019does} to reduce the accuracy of the knockoff student. The researchers found that "label smoothing encourages examples to lie in tight, equally separated clusters". Consequently, label smoothing reduces the contextual information in the teacher model's output \citep{10619241}.

\section{Power Transformation of Mutual Information}
\label{app:proofderiv}

The following theorem implies that for each label $y$, $ \mathrm{I}(X; \hat{Y}^{\alpha}|Y=y)   $ as a function of $\alpha$ is continuously differentiable. 
\begin{theorem} \label{th:deriv}
Let $(X,Z)$ be a pair of random variables, where $Z$ is discrete, and $X$ can be either discrete or continuous. Let $P_{Z|X}[\cdot|x]$ denote the conditional probability distribution of $Z$ given $X=x$. Additionally, let $P^{\alpha}_{Z|X}[\cdot|x]$ denote the power transformed version of $P_{Z|X}[\cdot|x]$ with power $\alpha$, $Z^{\alpha}$ denote the random variable the conditional distribution of which given $X=x$ is $P_{Z|X}^{\alpha}[\cdot|x]$, and $q^{\alpha}$ denote the probability distribution of $Z^{\alpha}$. Then, the following holds 
\begin{align} \label{eq:partialalpha}
\frac{\partial ~\mathrm{I} (X; Z^\alpha)}{\partial ~ \alpha} = \frac{1}{\alpha}  \sum_x P_X [x] \Big\{ \big( m_2(P_{Z|X}^{\alpha} [\cdot |x ]) - m_1^2 (P_{Z|X}^{\alpha} [\cdot |x ])\big)  - \text{Cov} (P_{Z|X}^{\alpha} [\cdot | x] ,q^\alpha)    \Big\}, 
\end{align} 
where for probability vectors $P$ and $Q$,   
% \small
\begin{subequations}\label{eq:ms}
\begin{align}
m_1(P) &\defeq  \sum_j P[j] \big( -\ln (P[j] )\big) = \mathsf{H}(P), \quad \text{(Shannon entropy)}   \\
m_2(P) &\defeq  \sum_j P[j] \big( -\ln (P[j] )\big)^2 , \quad \text{(Second moment)}\\
\text{Cov} (P,Q) &\defeq \sum_j P[j] \Big( -\ln(P[j]) - m_1(P) \Big) \Big( -\ln (Q[j]) - \sum_i P[i] (-\ln (Q[i]) )\Big).
\end{align}
\end{subequations}
% \normalsize

\end{theorem} 

\begin{proof}
To simplify our notation, we denote the conditional distributions $  P_{Z|X}[\cdot |x]  $ and $ P^\alpha_{Z|X}[j |x]   $ by $p_x$ and $p_x^{\alpha}$, respectively. Decompose $\mathrm{I} (X; Z^\alpha) $ as follows
\begin{align}
\mathrm{I} (X; Z^\alpha) & = H(Z^{\alpha}) - H (Z^{\alpha} | X) \nonumber \\
   & = \mathrm{H} (q^\alpha) - \sum_x P[x] \mathrm{H}(p_x^\alpha) \label{eq:th1-1}
\end{align}
where for any random variables $U$ and $V$, $H(V)$ and $H(V|U)$ are the entropy of $V$ and the conditional entropy of $V$ given $U$, respectively, and $H(p)$ denotes the entropy of the probability distribution $p$. Then the partial derivative of $\mathrm{I} (X; Z^\alpha)$ w.r.t. $\alpha$ is equal to
\begin{align} \label{app:eq1}
\frac{\partial \mathrm{I} (X; Z^\alpha)}{\partial \alpha} =  \frac{\partial \mathrm{H} (q^\alpha)}{\partial \alpha}  - \sum_x P[x]  \frac{\partial \mathrm{H}(p_x^\alpha)}{\partial \alpha}. 
\end{align}

To continue, we first compute the partial derivative in the second term of  the RHS of \eqref{app:eq1}
\small
\begin{align}
&\frac{\partial \mathrm{H}(p_x^\alpha)}{\partial \alpha} =   \frac{-\partial \sum_j p_x^\alpha [j ] \ln (p_x^\alpha [j])}{\partial \alpha} = - \sum_j \Big( \ln (p_x^\alpha [j ])+1 \Big) \frac{\partial p_x^{\alpha} [j ] }{\partial \alpha} \nonumber \\
&= - \sum_j \Big(\ln (p_x^\alpha [j ])+1 \Big) \nonumber \\
& \times \frac{(p_x [j ])^\alpha \ln (p_x [j]) \Big( \sum_i (p_x [i ])^\alpha \Big) -(p_x [j ])^\alpha \Big( \sum_i (p_x [i ])^\alpha \ln p_x [i] \Big)} {\Big( \sum_i (p_x [i ])^\alpha \Big)^2}  \nonumber \\
&=  - \sum_j \Big(\ln (p_x^\alpha [j ])+1 \Big) \Big(p_x^\alpha [j ] \big( \ln ( p_x [j ]) -\sum_i p_x^\alpha [i ] \ln (p_x [i]) \big) \Big)  \nonumber \\
&= \frac{-1}{\alpha} \sum_j \Big(\ln (p_x^\alpha [j ])+1 \Big) p_x^\alpha [j ] \Big( \ln (p_x [j ])^\alpha - \sum_i p_x^\alpha [i |] \ln (p_x [i ])^\alpha  \Big)  \label{eq:app1-}\\
&= \frac{-1}{\alpha} \sum_j \Big(\ln (p_x^\alpha [j ])+1 \Big) p_x^\alpha [j ] \Big( \ln (p_x^\alpha [j ]) - \sum_i p_x^\alpha [i] \ln (p_x^\alpha [i ])  \Big)  \nonumber \\
&= \frac{-1}{\alpha} \Bigg( \sum_j p_x^\alpha [j ] \big( \ln (p_x^\alpha [j ]) \big)^2 - \Big( \sum_j p_x^\alpha [j ] \ln (p_x^\alpha [j ]) \Big) \Big( \sum_i p_x^\alpha [i ] \ln (p_x^\alpha [i ]) \Big) \Bigg)  \nonumber \\ \label{eq:app1}
&= \frac{-1}{\alpha} \Big( m_2(p_x^\alpha ) - m^2_1(p_x^\alpha ) \Big)
\end{align}
\normalsize

Note that 
  \[ q^{\alpha} = \sum_x P[x] p_x^{\alpha}.\]
Then we have
% \small
\begin{align}
&\frac{\partial \mathrm{H} (q^\alpha)}{\partial \alpha}  =\frac{-\partial \sum_j q^\alpha[j] \ln(q^\alpha[j])}{\partial \alpha} =-  \sum_j \Big(\ln(q^\alpha[j]) +1 \Big) \frac{\partial q^\alpha[j]}{\partial \alpha} \nonumber \\ 
& = -  \sum_j \Big(\ln(q^\alpha[j]) +1 \Big) \sum_x P[x] \frac{\partial p_x^\alpha[j ] }{\partial \alpha} \nonumber  \\ 
& = \frac{-1}{\alpha}  \sum_j \Big(\ln(q^\alpha[j]) +1 \Big) \sum_x P[x] p_x^\alpha[j ] \Big( \ln (p_x^\alpha[j ]) + m_1(p_x^\alpha) \Big) \label{eq:app2-} \\ 
& = \frac{-1}{\alpha}  \sum_j \Big(\ln(q^\alpha[j]) \Big) \sum_x P[x] p_x^\alpha[j ] \Big( \ln (p_x^\alpha [j]) + m_1(p_x^\alpha) \Big) \nonumber  \\ 
& = \frac{-1}{\alpha}   \sum_x P[x]  \sum_j \Big(\ln(q^\alpha[j]) \Big) p_x^\alpha [j] \Big( \ln (p_x^\alpha [j ]) + m_1(p_x^\alpha) \Big)  \nonumber \\ 
& = \frac{-1}{\alpha}   \sum_x P[x] \Bigg(  \sum_j  p_x^\alpha [j ] \ln (p_x^\alpha [j]) \Big(\ln(q^\alpha[j]) \Big)  \nonumber \\
&\qquad \qquad - m_1(p_x^\alpha)  \sum_j  p_x^\alpha [j] \Big( - \ln(q^\alpha[j]) \Big)   \Bigg) \nonumber \\ \label{eq:app2}
&= \frac{-1}{\alpha} \sum_x P[x]~ \text{Cov} \left( p_x^\alpha , q^\alpha \right)
\end{align}
where \eqref{eq:app2-} is due to \eqref{eq:app1-}.
% \normalsize

From \Cref{eq:app1,eq:app2}, \Cref{th:deriv} follows. 
\end{proof}

\section{Proof of \Cref{th2}} \label{app:theorem2}

\Cref{th2} follows from \Cref{th:deriv} and the following lemma. 

\begin{lemma} \label{le1}
Let $g(t)$ be a continuously differentiable function over $[0, \beta]$. Then the following holds:
\begin{equation}
     \max_{t} g(t)   = \lim_{\omega \to \infty} \frac{1}{\omega} \ln \frac{1}{\beta} \int_0^{\beta} \exp{ \{\omega 
     g(t) \} } d t . \label{eq:le1}
 \end{equation}   
\end{lemma}

\begin{proof}
Let $t^*$ be an optimal point at which 
  \[ g(t^*) =  \max_{t} g(t) .\]
For any $ \epsilon >0 $, let ${\cal N} (t^*, \epsilon)$ denote a closed interval containing $t^*$ with length $\epsilon$. It is easy to verify that
 \[\frac{1}{\omega} \ln \frac{1}{\beta} \int_0^{\beta} \exp{ \{\omega 
     g(t) \} } d t  \leq g(t^*) \]
which implies that
 \begin{equation} \label{eq:le2}
\limsup_{\omega \to \infty}  \frac{1}{\omega} \ln \frac{1}{\beta} \int_0^{\beta} \exp{ \{\omega 
     g(t) \} } d t \leq g(t^*). 
 \end{equation}
On the other hand,
 \begin{align}
 \frac{1}{\omega} \ln \frac{1}{\beta} \int_0^{\beta} \exp{ \{\omega 
     g(t) \} } d t & \geq  \frac{1}{\omega} \ln \frac{1}{\beta} \int_{{\cal N}(t^*, \epsilon)} \exp{ \{\omega 
     g(t) \} } d t  \nonumber \\
     & \geq \frac{1}{\omega} \ln \frac{\epsilon}{\beta} \exp{ \{\omega \min_{t \in {\cal N} (t^*, \epsilon)} 
     g(t) \} } \nonumber \\
     & = \min_{t \in {\cal N} (t^*, \epsilon)}  g(t) + \frac{1}{\omega} \ln \frac{\epsilon}{\beta} . \label{eq:le3}  
 \end{align}
Letting $\omega \to \infty$ in \eqref{eq:le3} yields
 \begin{align}
     \liminf_{\omega \to \infty}  \frac{1}{\omega} \ln \frac{1}{\beta} \int_0^{\beta} \exp{ \{\omega 
     g(t) \} } d t \geq \min_{t \in {\cal N} (t^*, \epsilon)}  g(t). \label{eq:le4}
 \end{align}
Note that \eqref{eq:le4} is valid for any $\epsilon >0$. Letting $\epsilon \to 0$ in \eqref{eq:le4}, we have
 \begin{align}
     \liminf_{\omega \to \infty}  \frac{1}{\omega} \ln \frac{1}{\beta} \int_0^{\beta} \exp{ \{\omega 
     g(t) \} } d t \geq   g(t^*). \label{eq:le5}
 \end{align}
Then \eqref{eq:le1} follows from \eqref{eq:le2} and \eqref{eq:le5}. This completes the proof of \Cref{le1}. 
\end{proof}

\section{Datasets description} \label{sec:dataset}
\begin{itemize}
    \item CIFAR-100 \citep{CIFAR100} dataset contains 50K training and 10K test color images, each with size $32\times32$, categorized into 100 classes.

    \item TinyImageNet \citep{tinyiamgenet} contains 120K color images across 200 classes, each with a resolution of $64 \times 64$ pixels. For each class, there are 500 training images, 50 validation images and 50 test images.

    \item ImageNet \citep{imagenet} is a large-scale dataset used in visual recognition tasks, containing around 1.2 million training and 50K validation images.
\end{itemize}

\section{Experiments setup} \label{Appendix:ExperimentsSetup}
All experiments detailed in this paper were conducted using a publicly available national high-performance computer. For each experiment, we utilized 16 CPU cores, 64 GB of memory, and one NVIDIA V100 GPU. The software environment comprised Python 3.10, PyTorch 1.13, and CUDA 11.

For all experiments, including defenses and attacks, the SGD optimizer \citep{SGD, Efficientbackprop} with a learning rate of 0.1 is used unless otherwise specified.

For the CIFAR-100 and TinyImageNet datasets, we train the model for 200 epochs, decaying the learning rate by 0.1 at epochs {60, 120, 160}.

For ImageNet, we follow the standard PyTorch practice \footnote{\url{https://github.com/pytorch/vision/tree/main/references/classification}}. 

The batch size is 128 for both CIFAR-100 and TinyImageNet, and 256 for ImageNet.

To get the accuracy that a knockoff student can achieve using label smoothing, we have tested a wide spectrum of label smoothing factor $\epsilon=\{0.01, 0.05, 0.1, 0.15, 0.2, 0.25, 0.3, 0.4, 0.5, 0.6, 0.7, 0.8, 0.9\}$, and selected the value that yielded a classification accuracy exceeding that of all knockoff students.

In the CMIM method, we set $T = 20$ and tested $\lambda = {0.1, 0.25, 0.5, 1}$, selecting the value that minimized the CMI value while maintaining or improving classification accuracy.

\subsection{Defense setup} \label{Appendix:DefenseSetup}
We used the following parameters and settings for the defense models used in \Cref{sec:exp}.

    \subsubsection{Defense setup on CIFAR-100 and TinyImagenet}
    \begin{enumerate}
        \item \textbf{MAD}: 
        We employ a randomly initialized VGG-8 as adversary's architecture, and following the implementation of MAD-argmax.
        \item \textbf{APGP}: We apply a 3 layer MLP with residual connection as the generative model and set $\lambda = 0.1$ for all experiments.
        \item \textbf{RSP}: We use $\alpha = 1$ and $\beta=20$ for all the experiments.
        \item \textbf{NT}: To ensure a acceptable accuracy sacrifice, we test three different $\epsilon$ values and select the largest one that results in an accuracy drop of less than $0.5\%$. Specifically, we use $\epsilon = 0.01$ for ResNet-50 and $\epsilon = 0.005$ for VGG-16 on the CIFAR-100 dataset, while for TinyImageNet, we use $\epsilon = 0.001$  for both ResNet-34 and ResNet-50.
        \item \textbf{SNT}: We use the pretrained word2vect model namely "fasttext-wiki-news-subwords-300" provided by Gensim \citep{rehurek2011gensim}, and set $\lambda = 0.2$ for all experiments.
        \item \textbf{ST}: We use the teacher model trained by CE as the underlying model, and use the sparse ratio of $10\%$ as suggested in their paper for all experiments.
        \item \textbf{LS}: We apply label smoothing factor 0f 0.05 for all experiments.
        \item \textbf{ISTM}: We set the binary search parameters to $T_b =20$ and $\boldsymbol{\alpha}_{\rm max}=2000$. We use $\lambda = 0.2$ for ResNet-50 and $\lambda = 0.5$ for VGG-16 on the CIFAR-100 dataset, while for TinyImageNet, we use $\lambda = 0.1$  for ResNet-34 and $\lambda = 0.5$ ResNet-50.
    \end{enumerate}
    \subsubsection{Defense setup on Imagenet}
    \begin{enumerate}
        \item \textbf{ST}: We use the teacher model trained by CE as the underlying model, and use the sparse ratio of $10\%$ as suggested in their paper for all experiments.
        \item \textbf{ISTM}:  We set the binary search parameters to $T_b =20$ and $\boldsymbol{\alpha}_{\rm max}=2000$. We use $\lambda = 0.2$ for all the experiments.
    \end{enumerate}

\subsection{Attack setup}
    \subsubsection{Attack setup on CIFAR-100 and TinyImagenet}
    We use power transform parameter $\alpha = 0.25$ (or equivalently $T=4$) for all experiments unless otherwise specified.
    \begin{enumerate}
        \item \textbf{KD}: We set the CE-KL trade-off coefficient to $\lambda = 0.9$.
        \item \textbf{MKD}: We use the intrinsic dimension of 3 for CIFAR-100, and 5 for TinyImageNet. We employed the Adam optimizer \citep{kingma2014adam} with learning rate $10^{-3}$ for the trainable Markov transform.
        \item \textbf{DKD}: We test alpha, beta pairs of $\{1,4\}$ and $\{2,8\}$, and report the one with best accuracy.
        \item \textbf{DIST}: We $use \beta=1.0, \gamma=1.0, \tau=1.0$ for all experiments.
        \item \textbf{HTC}: We use $\alpha = 0.05 (T=20)$, $\lambda = 0.01$ for all experiments.
        \item \textbf{AVG}: $\lambda = 0.9$.
        \item \textbf{Knockoff}: We follow the implementation of the original paper.
    \end{enumerate}
    \subsubsection{Attack setup on Imagenet}
    We use $\alpha = 1$ ($T=1$) for all experiments unless otherwise specified.
    \begin{enumerate}
        \item \textbf{KD}: $\lambda = 0.9$.
        \item \textbf{MKD}: We use the intrinsic dimension of 16. We employed the Adam optimizer \citep{kingma2014adam} with learning rate $10^{-3}$ for the trainable Markov transform.
        \item \textbf{DKD}: We test alpha, beta pairs of $\{1,4\}$ and $\{2,8\}$, and report the one with best accuracy.
    \end{enumerate}

\section{Accuracy of Protected models} \label{sec:additional}
In this section, we report the top-1 accuracy of some additional models that are trained using CMIM and compare them with those trained by CE method. To this end, we use 10 well-known models for CIFAR-100 dataset namely ResNet (RN)-$\{18, 34, 50, 101, 152\}$, SqueezeNet (SQN), ResNext (RNXT) 50, MobileNet (MN), Xception (XCP), DenseNet (DN) 121; and 2 models namely RN-$\{34, 50\}$ for TinyImageNet and ImageNet. We follow the same training recipe as the one in \cref{exp:acck}. The results for CIFAR-100 and (Tiny-)ImageNet are listed in \cref{Tab:Cifar100ProtectedModelsGAin} and
\cref{Tab:ImgNetProtectedModelsGAin}, respectively. As seen, the top-1 accuracy for all models trained by CMIM is consistently higher than those trained by CE counterpart, with the gain up to $1.15\%$.

\begin{table}[h!]
\centering
\caption{Top-1 accuracy (\%) of models trained by CE and CMIM methods on CIFAR-100.}
\vskip -0.1in
\label{Tab:Cifar100ProtectedModelsGAin}
\begin{tabular}{|cccccc|}
\hline
\rowcolor{mygray} \multicolumn{6}{|c|}{CIFAR-100}   \\ \toprule \hline
\multicolumn{1}{|c|}{Model} & CE    & \multicolumn{1}{c|}{CMIM}  & \multicolumn{1}{c|}{Model}       & CE    & CMIM  \\ \hline
\multicolumn{1}{|c|}{RN18}  & 76.05 & \multicolumn{1}{c|}{77.20} & \multicolumn{1}{c|}{SQN}  & 69.32 & 70.64 \\
\multicolumn{1}{|c|}{RN34}  & 77.20 & \multicolumn{1}{c|}{77.54} & \multicolumn{1}{c|}{RNXT50}     & 78.71 & 79.12 \\
\multicolumn{1}{|c|}{RN50}  & 77.81 & \multicolumn{1}{c|}{77.93} & \multicolumn{1}{c|}{MN}   & 67.26 & 67.51 \\
\multicolumn{1}{|c|}{RN101} & 79.07 & \multicolumn{1}{c|}{79.12} & \multicolumn{1}{c|}{XCP}    & 77.37 & 77.64 \\
\multicolumn{1}{|c|}{RN152} & 79.21 & \multicolumn{1}{c|}{79.43} & \multicolumn{1}{c|}{DN121} & 79.16 & 79.33 \\ \hline
\end{tabular}
\end{table}

\begin{table}[h!]
\centering
\caption{Top-1 accuracy (\%) of models trained by CE and CMIM methods on TinyImageNet and ImageNet.}
\vskip -0.1in
\label{Tab:ImgNetProtectedModelsGAin}
\begin{tabular}{|ccc|ccc|}
\hline
\rowcolor{mygray}\multicolumn{3}{|c|}{TinyImageNet}          & \multicolumn{3}{c|}{ImageNet}              \\ \toprule \hline
\multicolumn{1}{|c|}{Model} & CE    & CMIM  & \multicolumn{1}{c|}{Model} & CE    & CMIM  \\ \hline
\multicolumn{1}{|c|}{RN34}  & 65.39 & 65.99 & \multicolumn{1}{c|}{RN34}  & 73.31 & 73.69 \\
\multicolumn{1}{|c|}{RN50}  & 66.14 & 66.93 & \multicolumn{1}{c|}{RN50}  & 76.15 & 76.40 \\ \hline
\end{tabular}
\end{table}

\clearpage
\section{Variance of \cref{Tab:BenchmarkCIFAR100}}
\label{app:variance}
\begin{table}[h!]
% \vspace{-0.2in}
\caption{Top-1 accuracy (\%) and variance of the knockoff student on CIFAR-100 and TinyImageNet dataset (averaged over 3 runs)}
\vskip -0.1in
\label{Tab:VarCIFAR100}
\begin{center}
\resizebox{1\textwidth}{!}{
\begin{tabular}{|cccccccccccc|}
\hline
\rowcolor{mygray} \multicolumn{11}{|c|}{CIFAR-100} \\ \toprule\hline
% \multicolumn{11}{|c|}{Post-Process Defence} \\ \hline
\multicolumn{1}{|c|}{Defense}                  & \multicolumn{1}{c|}{Model}                  & \multicolumn{1}{c|}{K-student} & LS & KD    & MKD   & DKD   & DIST  & HTC   & AVG   & \multicolumn{1}{c|}{Knockoff}  \\ \hline
\multicolumn{1}{|c|}{\multirow{4}{*}{MAD}}    & \multicolumn{1}{c|}{\multirow{2}{*}{VGG16}} & \multicolumn{1}{c|}{VGG11}   & 71.94$\pm$0.09 & 68.55$\pm$0.20 & 72.08$\pm$0.29 & 53.32$\pm$0.38 & 69.21$\pm$0.09 & 71.19$\pm$0.03 & 70.03$\pm$0.07 & \multicolumn{1}{c|}{61.44$\pm$0.06 }  \\ \cline{3-3}
\multicolumn{1}{|c|}{}                        & \multicolumn{1}{c|}{}                       & \multicolumn{1}{c|}{SNV2}    & 72.65$\pm$0.18 & 72.50$\pm$0.11 & 72.46$\pm$0.13 & 7.64 \ \ $\pm$ 0.70 & 69.91 $\pm$ 0.18 & 71.37$\pm$0.24 & 72.86$\pm$0.18 & \multicolumn{1}{c|}{70.87$\pm$0.26 }  \\ \cline{2-11} 
\multicolumn{1}{|c|}{}                        & \multicolumn{1}{c|}{\multirow{2}{*}{RN50}}  & \multicolumn{1}{c|}{VGG11}   & 71.94$\pm$0.09 & 72.00$\pm$0.21 & 72.04$\pm$0.13 & 54.29$\pm$0.52 & 71.57$\pm$0.31 & 70.76$\pm$0.20 & 70.73$\pm$0.22 & \multicolumn{1}{c|}{61.73$\pm$0.23 } \\ \cline{3-3}
\multicolumn{1}{|c|}{}                        & \multicolumn{1}{c|}{}                       & \multicolumn{1}{c|}{RN18}    & 78.76$\pm$0.08 & 77.76$\pm$0.23 & 78.79$\pm$0.23 & 43.73$\pm$0.55 & 73.76$\pm$0.10 & 77.89$\pm$0.25 & 78.61$\pm$0.15 & \multicolumn{1}{c|}{73.92$\pm$0.19 }  \\ \hline
\multicolumn{1}{|c|}{\multirow{4}{*}{APGP}}   & \multicolumn{1}{c|}{\multirow{2}{*}{VGG16}} & \multicolumn{1}{c|}{VGG11}   & 71.94$\pm$0.09 & 71.92$\pm$0.19 & 72.27$\pm$0.21 & 27.24$\pm$0.54 & 69.25$\pm$0.14 & 70.08$\pm$0.23 & 72.01$\pm$0.20 & \multicolumn{1}{c|}{45.98$\pm$0.45 } \\ \cline{3-3}
\multicolumn{1}{|c|}{}                        & \multicolumn{1}{c|}{}                       & \multicolumn{1}{c|}{SNV2}    & 72.65$\pm$0.18 & 73.10$\pm$0.27 & 73.75$\pm$0.23 & 12.52$\pm$0.30 & 71.04$\pm$0.31 & 71.66$\pm$0.13 & 73.20$\pm$0.27 & \multicolumn{1}{c|}{9.48 \ \ $\pm$ 0.73 }  \\ \cline{2-11} 
\multicolumn{1}{|c|}{}                        & \multicolumn{1}{c|}{\multirow{2}{*}{RN50}}  & \multicolumn{1}{c|}{VGG11}   & 71.94$\pm$0.09 & 71.91 $\pm$ 0.17 & 72.11 $\pm$ 0.23 & 9.74 \ \ $\pm$ 0.86 & 69.48 $\pm$ 0.11 & 71.38 $\pm$ 0.25 & 71.92 $\pm$ 0.14 & \multicolumn{1}{c|}{34.71 $\pm$ 0.30 } \\ \cline{3-3}
\multicolumn{1}{|c|}{}                        & \multicolumn{1}{c|}{}                       & \multicolumn{1}{c|}{RN18}    & 78.76 $\pm$ 0.08 & 78.04 $\pm$ 0.21 & 79.06 $\pm$ 0.14 & 62.71 $\pm$ 0.29 & 77.32 $\pm$ 0.13 & 77.82 $\pm$ 0.13 & 77.90 $\pm$ 0.15 & \multicolumn{1}{c|}{2.57 \ \ $\pm$ 0.95 } \\ \hline

\multicolumn{1}{|c|}{\multirow{4}{*}{RSP}}    & \multicolumn{1}{c|}{\multirow{2}{*}{VGG16}} & \multicolumn{1}{c|}{VGG11}   & 71.94 $\pm$ 0.09 & 71.42 $\pm$ 0.24 & 72.04 $\pm$ 0.13 & 70.22 $\pm$ 0.19 & 70.80 $\pm$ 0.17 & 70.40 $\pm$ 0.17 & 71.56 $\pm$ 0.06 & \multicolumn{1}{c|}{31.04 $\pm$ 0.62 } \\ \cline{3-3}
\multicolumn{1}{|c|}{}                        & \multicolumn{1}{c|}{}                       & \multicolumn{1}{c|}{SNV2}    & 72.65 $\pm$ 0.18 & 73.55 $\pm$ 0.34 & 72.95 $\pm$ 0.36 & 67.45 $\pm$ 0.20 & 72.19 $\pm$ 0.44 & 71.46 $\pm$ 0.41 & 72.27 $\pm$ 0.27 & \multicolumn{1}{c|}{26.09 $\pm$ 0.40} \\ \cline{2-11} 
\multicolumn{1}{|c|}{}                        & \multicolumn{1}{c|}{\multirow{2}{*}{RN50}}  & \multicolumn{1}{c|}{VGG11}   & 71.94 $\pm$ 0.09 & 71.97 $\pm$ 0.17 & 72.01 $\pm$ 0.20 & 69.53 $\pm$ 0.12 & 72.18 $\pm$ 0.21 & 70.87 $\pm$ 0.14 & 70.85 $\pm$ 0.17 & \multicolumn{1}{c|}{46.68 $\pm$ 0.60 } \\ \cline{3-3}
\multicolumn{1}{|c|}{}                        & \multicolumn{1}{c|}{}                       & \multicolumn{1}{c|}{RN18}    & 78.76 $\pm$ 0.08 & 77.78 $\pm$ 0.09 & 77.79 $\pm$ 0.16 & 77.01 $\pm$ 0.09 & 78.88 $\pm$ 0.21 & 78.00 $\pm$ 0.26 & 78.13 $\pm$ 0.12 & \multicolumn{1}{c|}{55.86 $\pm$ 0.18 } \\ \hline

\multicolumn{1}{|c|}{\multirow{4}{*}{NT}}     & \multicolumn{1}{c|}{\multirow{2}{*}{VGG16}} & \multicolumn{1}{c|}{VGG11}   & 71.94 $\pm$ 0.09 & 71.40 $\pm$ 0.34 & 73.44 $\pm$ 0.16 & 71.47 $\pm$ 0.14 & 71.33 $\pm$ 0.18 & 70.77 $\pm$ 0.23 & 71.58 $\pm$ 0.09 & \multicolumn{1}{c|}{63.56 $\pm$ 0.16 } \\ \cline{3-3}
\multicolumn{1}{|c|}{}                        & \multicolumn{1}{c|}{}                       & \multicolumn{1}{c|}{SNV2}    & 72.65 $\pm$ 0.18 & 72.44 $\pm$ 0.43 & 72.70 $\pm$ 0.35 & 6.24 \ 
 \ $\pm$ 0.51 & 72.04 $\pm$ 0.19 & 70.75 $\pm$ 0.13 & 72.83 $\pm$ 0.20 & \multicolumn{1}{c|}{6.32 $\pm$ 0.26 } \\ \cline{2-11} 
\multicolumn{1}{|c|}{}                        & \multicolumn{1}{c|}{\multirow{2}{*}{RN50}}  & \multicolumn{1}{c|}{VGG11}   & 71.94 $\pm$ 0.09 & 72.01 $\pm$ 0.25 & 72.03 $\pm$ 0.19 & 71.55 $\pm$ 0.36 & 71.88 $\pm$ 0.31 & 70.16 $\pm$ 0.29 & 71.94 $\pm$ 0.18 & \multicolumn{1}{c|}{62.94 $\pm$ 0.24 } \\ \cline{3-3}
\multicolumn{1}{|c|}{}                        & \multicolumn{1}{c|}{}                       & \multicolumn{1}{c|}{RN18}    & 78.76 $\pm$ 0.08 & 78.41 $\pm$ 0.25 & 78.92 $\pm$ 0.14 & 79.26 $\pm$ 0.29 & 78.99 $\pm$ 0.14 & 77.94 $\pm$ 0.22 & 78.33 $\pm$ 0.05 & \multicolumn{1}{c|}{68.96 $\pm$ 0.18 } \\ \hline

\multicolumn{1}{|c|}{\multirow{4}{*}{SNT}}    & \multicolumn{1}{c|}{\multirow{2}{*}{VGG16}} & \multicolumn{1}{c|}{VGG11}   & 71.94 $\pm$ 0.09 & 72.06 $\pm$ 0.22 & 72.28 $\pm$ 0.12 & 4.92 \ \ $\pm$ 0.22 & 71.98 $\pm$ 0.18 & 70.60 $\pm$ 0.13 & 71.63 $\pm$ 0.10 & \multicolumn{1}{c|}{64.08 $\pm$ 0.19 } \\ \cline{3-3}
\multicolumn{1}{|c|}{}                        & \multicolumn{1}{c|}{}                       & \multicolumn{1}{c|}{SNV2}    & 72.65 $\pm$ 0.18 & 72.94 $\pm$ 0.41 & 73.17 $\pm$ 0.13 & 72.78 $\pm$ 0.20 & 72.22 $\pm$ 0.24 & 71.22 $\pm$ 0.18 & 72.74 $\pm$ 0.20 & \multicolumn{1}{c|}{6.22 \ \ $\pm$ 0.59 }  \\ \cline{2-11} 
\multicolumn{1}{|c|}{}                        & \multicolumn{1}{c|}{\multirow{2}{*}{RN50}}  & \multicolumn{1}{c|}{VGG11}   & 71.94 $\pm$ 0.09 & 72.02 $\pm$ 0.19 & 72.12 $\pm$ 0.39 & 72.32 $\pm$ 0.33 & 71.70 $\pm$ 0.39 & 70.66 $\pm$ 0.17 & 71.65 $\pm$ 0.20 & \multicolumn{1}{c|}{62.94 $\pm$ 0.29 } \\ \cline{3-3}
\multicolumn{1}{|c|}{}                        & \multicolumn{1}{c|}{}                       & \multicolumn{1}{c|}{RN18}    & 78.76 $\pm$  0.08 & 78.25 $\pm$ 0.05 & 78.48 $\pm$ 0.24 & 78.82 $\pm$ 0.30 & 78.14 $\pm$ 0.28 & 78.45 $\pm$ 0.15 & 78.38 $\pm$ 0.13 & \multicolumn{1}{c|}{67.71 $\pm$ 0.20 } \\ \hline

\multicolumn{1}{|c|}{\multirow{4}{*}{ST}}    & \multicolumn{1}{c|}{\multirow{2}{*}{VGG16}} & \multicolumn{1}{c|}{VGG11}    & 71.94 $\pm$ 0.09 & 72.09 $\pm$ 0.23 & 72.01 $\pm$ 0.14 & 71.63 $\pm$ 0.16 & 71.93 $\pm$ 0.12 & 71.16 $\pm$ 0.27 & 71.63 $\pm$ 0.18 & \multicolumn{1}{c|}{63.32 $\pm$ 0.14 } \\ \cline{3-3}
\multicolumn{1}{|c|}{}                        & \multicolumn{1}{c|}{}                       & \multicolumn{1}{c|}{SNV2}    & 72.65 $\pm$ 0.18 & 72.64 $\pm$ 0.15 & 72.67 $\pm$ 0.21 & 70.53 $\pm$ 0.48 & 72.24 $\pm$ 0.39 & 71.32 $\pm$ 0.38 & 72.42 $\pm$ 0.12 & \multicolumn{1}{c|}{69.46 $\pm$ 0.28} \\ \cline{2-11} 
\multicolumn{1}{|c|}{}                        & \multicolumn{1}{c|}{\multirow{2}{*}{RN50}}  & \multicolumn{1}{c|}{VGG11}   & 71.94 $\pm$ 0.09 & 72.00 $\pm$ 0.19 & 72.13 $\pm$ 0.13 & 71.62 $\pm$ 0.24 & 71.76 $\pm$ 0.29 & 70.54 $\pm$ 0.33 & 71.73 $\pm$ 0.11 & \multicolumn{1}{c|}{65.43 $\pm$ 0.24 } \\ \cline{3-3}
\multicolumn{1}{|c|}{}                        & \multicolumn{1}{c|}{}                       & \multicolumn{1}{c|}{RN18}    & 78.76 $\pm$ 0.08 & 78.96 $\pm$ 0.26 & 79.02 $\pm$ 0.06 & 78.35 $\pm$ 0.09 & 78.31 $\pm$ 0.14 & 78.36 $\pm$ 0.25 & 78.81 $\pm$ 0.14 & \multicolumn{1}{c|}{72.87 $\pm$ 0.08 } \\ \hline

\multicolumn{1}{|c|}{\multirow{4}{*}{LS}}     & \multicolumn{1}{c|}{\multirow{2}{*}{VGG16}} & \multicolumn{1}{c|}{VGG11}   & 71.94 $\pm$ 0.09 & 71.90 $\pm$ 0.18 & 72.00 $\pm$ 0.06 & 71.57 $\pm$ 0.26 & 70.89 $\pm$ 0.12 & 70.66 $\pm$ 0.17 & 71.76 $\pm$ 0.10 & \multicolumn{1}{c|}{63.49 $\pm$ 0.21 } \\ \cline{3-3}
\multicolumn{1}{|c|}{}                        & \multicolumn{1}{c|}{}                       & \multicolumn{1}{c|}{SNV2}    & 72.65 $\pm$ 0.18 & 72.87 $\pm$ 0.28 & 73.52 $\pm$ 0.25 & 70.01 $\pm$ 0.32 & 71.49 $\pm$ 0.38 & 71.70 $\pm$ 0.42 & 73.01 $\pm$ 0.27 & \multicolumn{1}{c|}{65.20 $\pm$ 0.14 }  \\ \cline{2-11} 
\multicolumn{1}{|c|}{}                        & \multicolumn{1}{c|}{\multirow{2}{*}{RN50}}  & \multicolumn{1}{c|}{VGG11}   & 71.94 $\pm$ 0.09 & 71.82 $\pm$ 0.28 & 71.99 $\pm$ 0.16 & 71.95 $\pm$ 0.33 & 70.77 $\pm$ 0.39 & 70.86 $\pm$ 0.24 & 71.88 $\pm$ 0.16 & \multicolumn{1}{c|}{62.29 $\pm$ 0.10 } \\ \cline{3-3}
\multicolumn{1}{|c|}{}                        & \multicolumn{1}{c|}{}                       & \multicolumn{1}{c|}{RN18}    & 78.76 $\pm$ 0.08 & 77.72 $\pm$ 0.30 & 77.82 $\pm$ 0.12 & 79.37 $\pm$ 0.19 & 78.33 $\pm$ 0.06 & 78.31 $\pm$ 0.21 & 77.91 $\pm$ 0.07 & \multicolumn{1}{c|}{63.36 $\pm$ 0.17 } \\ \hline

\multicolumn{1}{|c|}{\multirow{4}{*}{CMIM}} & \multicolumn{1}{c|}{\multirow{2}{*}{VGG16}}   & \multicolumn{1}{c|}{VGG11}   & 71.94 $\pm$ 0.09 & 71.87 $\pm$ 0.24 & 71.64 $\pm$ 0.07 & 71.56 $\pm$ 0.03 & 70.34 $\pm$ 0.09 & 71.71 $\pm$ 0.14 & 71.42 $\pm$ 0.05 & \multicolumn{1}{c|}{66.89 $\pm$ 0.11 } \\ \cline{3-3}
\multicolumn{1}{|c|}{}                        & \multicolumn{1}{c|}{}                       & \multicolumn{1}{c|}{SNV2}    & 72.65 $\pm$ 0.18 & 72.53 $\pm$ 0.21 & 71.44 $\pm$ 0.16 & 72.46 $\pm$ 0.20 & 71.45 $\pm$ 0.31 & 71.59 $\pm$ 0.24 & 71.94 $\pm$ 0.20 & \multicolumn{1}{c|}{64.45 $\pm$ 0.24 } \\ \cline{2-11} 
\multicolumn{1}{|c|}{}                        & \multicolumn{1}{c|}{\multirow{2}{*}{RN50}}  & \multicolumn{1}{c|}{VGG11}   & 71.94 $\pm$ 0.09 & 71.54 $\pm$ 0.30 & 71.34 $\pm$ 0.16 & 71.77 $\pm$ 0.06 & 71.86 $\pm$ 0.28 & 69.32 $\pm$ 0.09 & 71.70 $\pm$ 0.22 & \multicolumn{1}{c|}{60.58 $\pm$ 0.17 } \\ \cline{3-3}
\multicolumn{1}{|c|}{}                        & \multicolumn{1}{c|}{}                       & \multicolumn{1}{c|}{RN18}    & 78.76 $\pm$ 0.08 & 78.21 $\pm$ 0.13 & 78.16 $\pm$ 0.09 & 78.13 $\pm$ 0.06 & 77.56 $\pm$ 0.06 & 77.23 $\pm$ 0.09 & 78.64 $\pm$ 0.06 & \multicolumn{1}{c|}{65.88 $\pm$ 0.09 } \\  \toprule\hline

\rowcolor{mygray} \multicolumn{11}{|c|}{TinyImageNet} \\ \toprule\hline 

\multicolumn{1}{|c}{\multirow{2}{*}{RSP}} & \multicolumn{1}{|c|}{RN34} & \multicolumn{1}{|c|}{RN18}   & 63.56 $\pm$ 0.06 & 63.54 $\pm$ 0.09 & 64.32 $\pm$ 0.07 & 64.01 $\pm$ 0.07 & 63.27 $\pm$ 0.16 & 63.54 $\pm$ 0.07 & 62.15 $\pm$ 0.14 & \multicolumn{1}{c|}{55.43 $\pm$ 0.12 } \\ \cline{2-11}
                                          & \multicolumn{1}{|c|}{RN50} & \multicolumn{1}{|c|}{SNV2}   & 60.61 $\pm$ 0.15 & 60.18 $\pm$ 0.26 & 60.76 $\pm$ 0.16 & 56.26 $\pm$ 0.16 & 56.43 $\pm$ 0.11 & 60.96 $\pm$ 0.22 & 60.15 $\pm$ 0.20 & \multicolumn{1}{c|}{54.01 $\pm$ 0.22 }  \\ \cline{2-11} \hline

 \multicolumn{1}{|c}{\multirow{2}{*}{ST}} & \multicolumn{1}{|c|}{RN34} & \multicolumn{1}{|c|}{RN18}   & 63.56 $\pm$ 0.06 & 63.96 $\pm$ 0.13 & 64.12 $\pm$ 0.07 & 63.25 $\pm$ 0.10 & 63.51 $\pm$ 0.14 & 63.49 $\pm$ 0.19 & 63.84 $\pm$ 0.08 & \multicolumn{1}{c|}{57.42 $\pm$ 0.08 } \\ \cline{2-11}
                                          & \multicolumn{1}{|c|}{RN50} & \multicolumn{1}{|c|}{SNV2}   & 60.61 $\pm$ 0.15 & 61.23 $\pm$ 0.24 & 61.36 $\pm$ 0.14 & 60.43 $\pm$ 0.14 & 60.32 $\pm$ 0.24 & 60.22 $\pm$ 0.17 & 61.13 $\pm$ 0.13 & \multicolumn{1}{c|}{55.84 $\pm$ 0.11 } \\ \cline{2-11} \hline

\multicolumn{1}{|c}{\multirow{2}{*}{NT}}  & \multicolumn{1}{|c|}{RN34} & \multicolumn{1}{|c|}{RN18}   & 63.56 $\pm$ 0.06 & 63.27 $\pm$ 0.14 & 64.49 $\pm$ 0.17 & 64.67 $\pm$ 0.16 & 63.43 $\pm$ 0.20 & 63.50 $\pm$ 0.10 & 64.43 $\pm$ 0.11 & \multicolumn{1}{c|}{53.11 $\pm$ 0.06 }  \\ \cline{2-11}
                                          & \multicolumn{1}{|c|}{RN50} & \multicolumn{1}{|c|}{SNV2}   & 60.61 $\pm$ 0.15 & 59.57 $\pm$ 0.22 & 61.55 $\pm$ 0.12 & 31.55 $\pm$ 0.28 & 60.03 $\pm$ 0.23 & 60.98 $\pm$ 0.27 & 60.31 $\pm$ 0.17 & \multicolumn{1}{c|}{50.94 $\pm$ 0.15 } \\ \cline{2-11} \hline

\multicolumn{1}{|c}{\multirow{2}{*}{LS}}  & \multicolumn{1}{|c|}{RN34}  & \multicolumn{1}{|c|}{RN18}  & 63.56 $\pm$ 0.06 & 63.74 $\pm$ 0.08 & 64.01 $\pm$ 0.14 & 64.23 $\pm$ 0.11 & 63.51 $\pm$ 0.07 & 64.20 $\pm$ 0.16 & 63.04 $\pm$ 0.13 & \multicolumn{1}{c|}{57.43 $\pm$ 0.10 } \\ \cline{2-11}
                                          & \multicolumn{1}{|c|}{RN50}  & \multicolumn{1}{|c|}{SNV2}  & 60.61 $\pm$ 0.15 & 60.32 $\pm$ 0.24 & 60.93 $\pm$ 0.15 & 60.74 $\pm$ 0.26 & 60.11 $\pm$ 0.28 & 60.46 $\pm$ 0.14 & 60.14 $\pm$ 0.21 & \multicolumn{1}{c|}{52.96 $\pm$ 0.23 } \\ \cline{2-11} \hline

\multicolumn{1}{|c}{\multirow{2}{*}{CMIM}}& \multicolumn{1}{|c|}{RN34} & \multicolumn{1}{|c|}{RN18}   & 63.53 $\pm$ 0.06 & 62.89 $\pm$ 0.03 & 63.15 $\pm$ 0.08 & 62.94 $\pm$ 0.03 & 63.28 $\pm$ 0.05 & 61.57 $\pm$ 0.06 & 62.96 $\pm$ 0.06 & \multicolumn{1}{c|}{56.13 $\pm$ 0.04 } \\ \cline{2-11}
                                          & \multicolumn{1}{|c|}{RN50}  & \multicolumn{1}{|c|}{SNV2}  & 60.61 $\pm$ 0.15 & 57.57 $\pm$ 0.20 & 59.32 $\pm$ 0.17 & 60.58 $\pm$ 0.12 & 59.41 $\pm$ 0.09 & 59.33 $\pm$ 0.10 & 60.42 $\pm$ 0.04 & \multicolumn{1}{c|}{56.91 $\pm$ 0.05 } \\ \cline{2-11} \hline
\end{tabular}}
\vskip -0.4in
\end{center}
\end{table}

Upon reviewing the variance estimates, we see that certain cases, such as the (RN50, VGG11) pair on CIFAR-100, might suggest that the CMIM-trained model could be rendered distillable under naive statistical interpretations. For example, when the DIST method is applied to the CMIM model, the accuracy is reported as, which could potentially exceed the LS student accuracy of if variance is simply added to the mean.

However, this approach of directly comparing mean values with added variances does not provide an accurate or fair assessment of undistillability. To address this, we conducted a more rigorous analysis where, across five different seeds, we compared the accuracy of the knock-off Student (VGG11 in this case) trained via label smoothing and the DIST method applied to RN50 trained with CMIM.

The results of this comprehensive comparison are summarized in the following table, which demonstrates that the CMIM model achieves undistillability across all different seeds.

\begin{table}[h!]
\caption{Variance analysis of \cref{Tab:VarCIFAR100}}
\begin{center}
\vskip -0.2in
\resizebox{0.5\textwidth}{!}{\begin{tabular}{|c|c|c|c|c|c|}
\hline \rowcolor{mygray} 
& Seed 1& Seed 2 & Seed 3 & Seed 4 & Seed 5    \\  \toprule \hline
LS  & 72.28	&71.83	&72.06	&72.25	&71.53 \\
CMIM & 72.01	&71.74	&71.93&	72.12	&71.27 \\ \hline
\end{tabular}}
\end{center}
\end{table}

\section{Computational Overhead} \label{sec:computation}

In this section we compare the computational overhead of our method with the CE counterpart on CIFAR-100 dataset.

\begin{table}[h!]
\caption{Training times of CMIM and CE for ResNet-10 and VGG-16 on the CIFAR-100 dataset.}
\begin{center}
\vskip -0.2in
\begin{tabular}{|c|c|c|}
\hline \rowcolor{mygray} 
      & CE                 & CMIM               \\  \toprule \hline
RN50  & 4 hours 43 minutes & 5 hours 13 minutes \\
VGG16 & 2 hours 57 minutes & 3 hours 25 minutes \\ \hline
\end{tabular}
\end{center}
\end{table}

Note that the training time for CMIM is slightly higher than that of conventional CE method. This is primarily due to the additional inference samples required to estimate the CMI per class. We believe this is a reasonable trade-off given the significant benefits of the method. In addition, note that the number of samples $N$ does not have any effects on the training time CMIM; this is because the power transform applied to the teacher's output probabilities, and when calculating the gradients during the backpropagation, different values of $\alpha$
does not change the gradients. 

Additionally, we note that among the existing KD protection methods in the literature, only CMIM (our method) and ST are scalable to larger datasets like ImageNet. This scalability is due to the significant computational complexity of other benchmark methods, which limits their applicability to smaller datasets.

\section{Ablation on Hyper-parameters} \label{sec:abl}
In this section, we perform an ablation study to analyze the impact of three key hyper-parameters of CMIM: \(\beta\), the number of samples \(N\), and \(\omega\). 

For all experiments in this study, we employ the VGG-16 and SNV2 models as the teacher-student pair, using the CIFAR-100 dataset as the evaluation benchmark.

\subsection{Range of $\beta$}
In this section, we examine the impact of \(\beta\) on the performance of CMIM. For this analysis, we fix \(N=50\) and \(\omega=25\), while varying the value of \(\beta\). The results are shown in \cref{fig:beta} show that the accuracy of the knockoff student decreases significantly when \(\beta \geq 1\), highlighting 
% the sensitivity of the model to higher values of \(\beta\).
the importance of choosing an appropriate $\beta$ to effectively balance cross-entropy and CMI objectives.
\begin{figure}[h!]  
	\centering
	\includegraphics[width=\columnwidth]{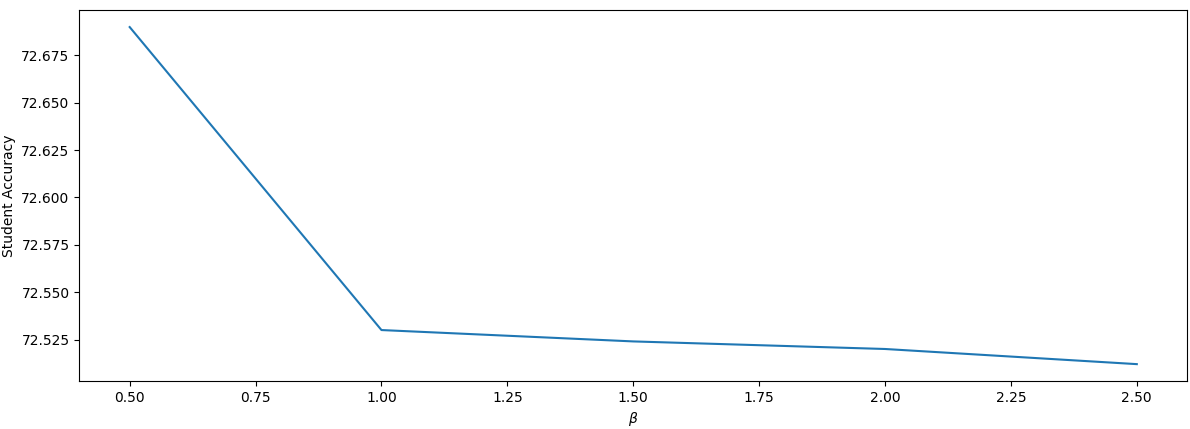}
    \caption{The student accuracy when distilled from the teacher model trained by different $\beta$ values.} \label{fig:beta}
\end{figure}

\subsection{Number of power samples $N$}
In this section, we analyze the influence of the number of samples \(N\) on the performance of the knockoff student. For this study, we set \(\beta = 2\) and \(\omega = 25\), varying \(N\) to observe its impact. The results, presented in \cref{fig:N}, reveal that the accuracy of the knockoff student declines monotonically as \(N\) increases, suggesting that a moderate sampling size is enable to capture sufficient diversity for robust estimation.

\begin{figure}[h!]  
	\centering
	\includegraphics[width=\columnwidth]{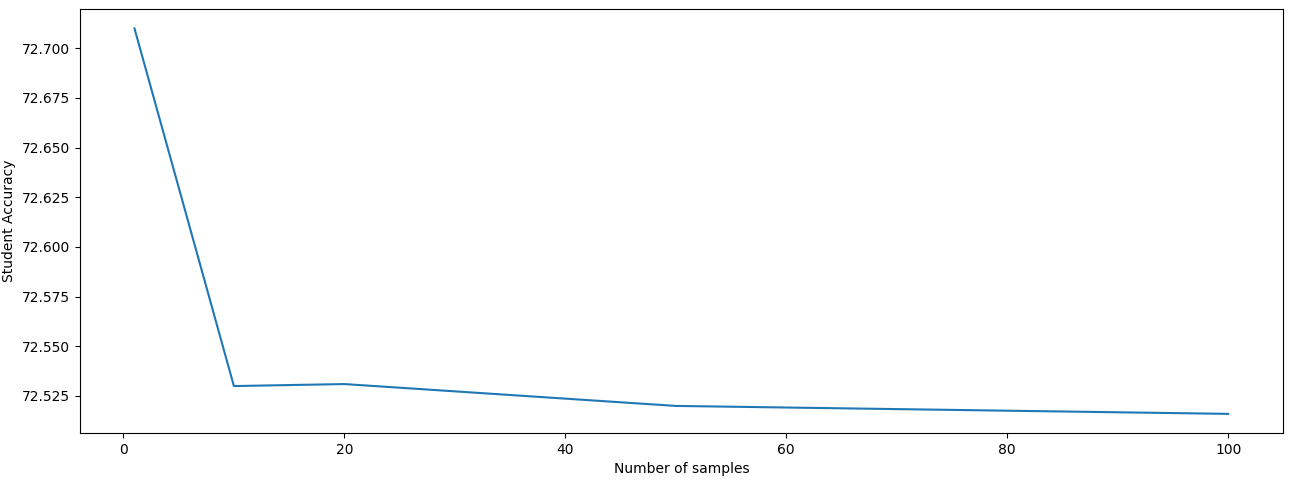}
    \caption{The student accuracy when distilled from the teacher model trained by different number of sample $N$.} \label{fig:N}
\end{figure}

\subsection{Power Coefficient $\omega$}
In this section, we investigate the effect of the power coefficient \(\omega\) on the knockoff student's performance. For this analysis, we fix \(\beta = 2\) and \(N = 25\), while varying \(\omega\). The results are summarized in \cref{tab:omega}. Notably, when \(\omega > 30\), the simulation frequently results in NaN values due to excessively large exponent values. Furthermore, the table shows that at \(\omega = 20\) or \(\omega = 30\), the knockoff student's accuracy reaches its minimum, indicating that this value effectively approximates the behavior of \(\omega = \infty\).

\begin{table}[!h]
\caption{The student accuracy when distilled from the teacher model trained by different values of $\omega$.} 
\vskip -0.1in
\resizebox{1\textwidth}{!}{\begin{tabular}{|c|l|l|l|l|l|l|l|l|l|l|l|l|}
\hline \rowcolor{mygray} 
Value of $\omega$          & 1     & 2     & 5     & 10    & 15    & 20    & 25    & 30    & 40  & 50    & 100 & 200 \\  \toprule \hline
Knock-off Student Accuracy & 72.65 & 72.67 & 72.55 & 72.56 & 72.55 & 72.52 & 72.53 & 72.52 & NaN & NaN & NaN & NaN \\ \hline
\end{tabular}} \label{tab:omega}
\end{table}

\clearpage
\section{Visualization of Different Three Projected Probability Clusters}
\label{Appendix:MoreSimplex}

\begin{figure*}[!ht] 
% \vskip -0.2in
	\centering
	\subfloat[LS]{\includegraphics[width=0.30\columnwidth]{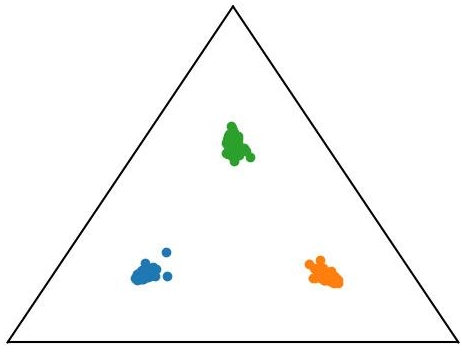}\label{fig:LS_a}} 
	\subfloat[Nasty teacher] {\includegraphics[width=0.30\columnwidth]{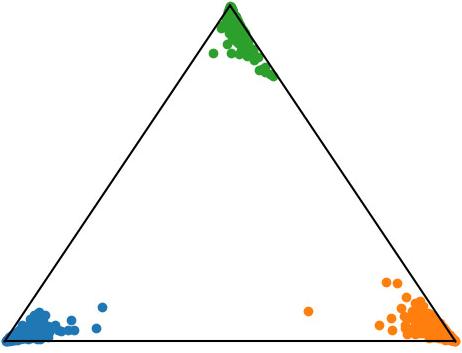}\label{fig:Nastya}}
 \subfloat[CMIM]{\includegraphics[width=0.30\columnwidth]{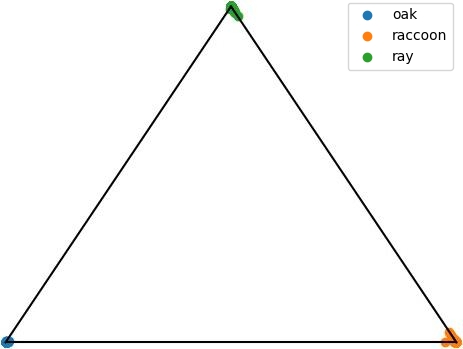}\label{fig:CMIMa}}
	\vskip -0.1in
    \caption{Visualization of three projected probability clusters for ResNet-50 trained on CIFAR-100 using (a) LS, (b) NT, and (c) CMIM.}
    \label{fig:Simplex}
	% \vskip -0.2in
\end{figure*}
In this section, we randomly selected three additional output clusters from the CIFAR-100 dataset, oak, raccoon, and ray, and visualized them using the method introduced in \cite{10900607}. The same conclusion continues to hold.

\section{Effect of CMIM on Model Calibration}
\label{Appendix:Caliberation}
 In this section, we investigate the impact of the CMIM method on model calibration. To do so, we report the Expected Calibration Error (ECE) for ResNet-50 trained with CE, LS, NT and CMIM on the CIFAR-100 dataset. As shown in Table \ref{Tab:Calibration}, the CMIM-trained model achieves calibration on par with its CE-trained counterpart and consistently outperforms both the LS and NT counterparts.

\begin{table}[!h]
\centering
\caption{Expected calibration error of three projected probability clusters for ResNet-50 trained on CIFAR-100 using CE, LS, NT and CMIM.}
% \vskip -0.1in
\label{Tab:Calibration}
\begin{tabular}{|c|cccc|}
\hline \rowcolor{mygray} 
Defense & CE     & LS      & NT      & CMIM    \\ \toprule\hline
ECE     & 9.42\% & 21.48\% & 16.47\% & 11.24\% \\ \hline 
Accuracy     & 77.81\% & 78.45\% & 77.31\% & 78.72\% \\ \hline 
\end{tabular}
\end{table}

\end{document}